\documentclass[10pt]{article} % For LaTeX2e
\usepackage[accepted]{tmlr}

\PassOptionsToPackage{table, svgnames, dvipsnames}{xcolor}

\usepackage{amsmath,amsfonts,bm}

\def\eqref#1{equation~\ref{#1}}
\def\1{\bm{1}}

\DeclareMathAlphabet{\mathsfit}{\encodingdefault}{\sfdefault}{m}{sl}
\SetMathAlphabet{\mathsfit}{bold}{\encodingdefault}{\sfdefault}{bx}{n}

\newcommand{\E}{\mathbb{E}}

\newcommand{\R}{\mathbb{R}}

\DeclareMathOperator*{\argmin}{arg\,min}

\usepackage{hyperref}
\usepackage{url}

\usepackage{amsmath}
\usepackage{amssymb}
\usepackage{mathtools}
\usepackage{amsthm}
\usepackage{url}

\usepackage{framed}
\usepackage{graphicx}
\usepackage{xcolor}
\usepackage{amsfonts}
\usepackage{algorithm}
\usepackage{algpseudocode}
\usepackage{subfigure}
\usepackage[shortlabels]{enumitem}
\usepackage{threeparttable}

\usepackage{makecell}
\usepackage{mathtools}
\usepackage{titlesec}
\usepackage{wrapfig}
\renewcommand{\cite}[1]{\citep{#1}}

\theoremstyle{plain}
\newtheorem{theorem}{Theorem}[section]

\newtheorem{lemma}[theorem]{Lemma}

\theoremstyle{definition}
\newtheorem{definition}[theorem]{Definition}
\newtheorem{assumption}[theorem]{Assumption}
\theoremstyle{remark}
\newtheorem{remark}[theorem]{Remark}

\DeclareMathOperator*{\Rd}{\mathbb{R}^d}
\DeclareMathOperator*{\PP}{\mathbb{P}}
\DeclareMathOperator*{\cO}{\mathcal{O}}

\DeclareMathOperator{\cG}{\mathcal{G}}

\DeclareMathOperator*{\cS}{\mathcal{S}}
\DeclareMathOperator*{\cA}{\mathcal{A}}
\DeclareMathOperator*{\cP}{\mathcal{P}}
\DeclareMathOperator*{\cR}{\mathcal{R}}
\DeclareMathOperator{\g}{\gamma}

\algnewcommand{\LineComment}[1]{\State \(\triangleright\) #1}

\DeclarePairedDelimiterX\Set[1]{\{}{\}}{
  
  #1
}

\newcommand{\norm}[1]{\left \lVert #1 \right\rVert }

\usepackage[textsize=tiny]{todonotes}

\title{Global Convergence Guarantees for Federated Policy Gradient Methods with Adversaries}

\def\month{11}  % Insert correct month for camera-ready version
\def\year{2024} % Insert correct year for camera-ready version
\def\openreview{\url{https://openreview.net/forum?id=Ea0LrPORzM}} % Insert correct link to OpenReview for camera-ready version

\author{\name Swetha Ganesh
        \email swethaganesh@iisc.ac.in \\ 
		\addr 
        Indian Institute of Science (IISc), 
        Bengaluru 560012, India \\
		Purdue University, 
		West Lafayette, IN, 47907, USA 
		\AND
		\name Jiayu Chen 
        \email jiayuc2@andrew.cmu.edu \\
		\addr Carnegie Mellon University (CMU), Pittsburgh, PA, 15289, USA
        \AND
		\name Gugan Thoppe  
        \email gthoppe@iisc.ac.in \\
        \addr 
        Indian Institute of Science (IISc), 
        Bengaluru 560012, India	
        \AND
		\name Vaneet Aggarwal 
        \email vaneet@purdue.edu \\
		\addr 
		Purdue University, 
		West Lafayette, IN, 47907, USA		
	}

\begin{document}
\maketitle

\begin{abstract}
Federated Reinforcement Learning (FRL) allows multiple agents to collaboratively build a decision making policy without sharing raw trajectories. However, if a small fraction of these agents are adversarial, it can lead to catastrophic results. We propose a policy gradient based approach that is robust to adversarial agents which can send arbitrary values to the server. Under this setting, our results form the first global convergence guarantees with general parametrization. These results demonstrate resilience with adversaries, while achieving optimal sample complexity of order $\tilde{\mathcal{O}}\left( \frac{1}{N\epsilon^2} \left( 1+ \frac{f^2}{N}\right)\right)$, where $N$ is the total number of agents and $f<N/2$ is the number of adversarial agents. %This translates into a linear speedup in $N$ when $f \leq \mathcal{O}(\sqrt{N})$.
%\gt{Elaborate on what general parametrization means}
\end{abstract}

\section{Introduction}

%Intro to RL and Policy Gradient

Reinforcement Learning (RL) encompasses a category of challenges in which an agent iteratively selects actions and acquires rewards within an unfamiliar environment, all while striving to maximize the cumulative rewards. The Policy Gradient (PG) based approaches serve as an effective means of addressing RL problems, demonstrating its successful application in a diverse range of complex domains, such as  game playing \cite{schrittwieser2020mastering,bonjour2022decision}, transportation \cite{kiran2021deep,al2019deeppool}, robotics \cite{abeyruwan2023sim2real,chen2023option}, telesurgery \cite{gonzalez2023asap}, network scheduling \cite{geng2020multi,chen2023hybrid}, and healthcare \cite{yu2021reinforcement}. 

%recent examples including  InstructGPT and ChatGPT (GPT-4). 

% Too little pages right now, should i expand on RL?

% Intro to FedRL
%Typically, RL applications require a very large amount of training data in order to achieve a desirable target accuracy. Thus, parallelizing training data collection can provide a significant speed up. One way of doing this is through Federated Reinforcement Learning (FRL). In Federated Learning, workers communicate locally trained models rather than raw datasets. This ensures communication efficiency and data privacy. Though FL is commonly used in supervised learning problems, recent works have extended its use to FRL. Here, multiple agents collaboratively build a decision making policy without sharing raw trajectories. 

RL applications often demand extensive training data to attain the desired accuracy level. Parallelizing training can significantly expedite this process, with one approach being Federated Reinforcement Learning (FRL) \cite{jin2022federated}. In FRL, workers exchange locally trained models instead of raw data, ensuring efficient communication and data privacy. While Federated Learning (FL) is typically associated with supervised learning \cite{9311906}, recent developments have extended its application to FRL, enabling multiple agents to collaboratively construct decision-making policies without sharing raw trajectories.

%Intro to Byzantine Fault

Distributed systems, including FRL, face vulnerabilities such as random failures or adversarial attacks. These issues may arise from agents' arbitrary behavior due to hardware glitches, inaccurate training data, or deliberate malicious attacks. In the case of attacks, the possibility exists that attackers possess extensive knowledge and can collaborate with others to maximize disruption. These scenarios fall under the Byzantine failure model, the most rigorous fault framework in distributed computing, where a minority of agents can act arbitrarily and potentially maliciously with the aim of disrupting the system's convergence \cite{lamport1982}. In this paper, we investigate how adversaries affect the overall convergence of federated policy gradient methods.

%However, distributed systems including FRL are vulnerable to random failures or adversarial attacks. These can occur when some agents behave arbitrarily due to random hardware issues, inaccurate training data, or adversarial/malicious attacks. In the case of attacks, it is possible that the attackers are omniscient and can collude with other attackers to cause maximum disruption. The scenarios described above can be viewed as special cases under the Byzantine failure model, which is considered as the most stringent fault formalism in distributed computing. In this model, a small fraction
%of agents may behave arbitrarily and possibly adversarially, with the intention of interfering with the convergence of the system. 

%It is known that the convergence is possible only when the fraction of adversarial agents is less than 1/2. 

\textbf{Related Works:} 

{\bf 1.  Global Convergence of Policy Gradient Approaches: } Recently, there has been an increasing research emphasis on exploring the global convergence of PG-based methods, going beyond the well-recognized convergence to first-order stationary policies. In \cite{agarwal2021theory}, a fairly comprehensive characterization of global convergence for PG approaches is provided. Additionally, other significant studies on sample complexity for global convergence include \cite{wang2019neural,xu2019sample,liu2020improved,masiha,fatkhullin,mondal2023improved}. We note that both \citep{fatkhullin} and \citep{mondal2023improved} achieve near-optimal sample complexity of $\tilde{O}(1/\epsilon^2)$. In the absence of adversarial elements, we use the work in \cite{fatkhullin} as a foundational reference for this paper.

{\bf 2.  Federated Reinforcement Learning: } Federated reinforcement learning has been explored in various setups, including tabular RL \cite{agarwal2021communication, jin2022federated, khodadadian2022federated}, control tasks \cite{wang2023model}, and value-function based algorithms \cite{wang2023federated, fedkl2023}, showcasing linear speedup. In the case of policy-gradient based algorithms, it's evident that linear speedup is achievable, as each agent's collected trajectories can be parallelized \cite{lan2023improved}. Nevertheless, achieving speedup becomes challenging with increasing nodes when adversaries are introduced, and this paper focuses on addressing this key issue.

{3. \bf Byzantine Fault Tolerance in Federated Learning: } %{\bf Add some key works here. } 
A significant body of research has focused on distributed Stochastic Gradient Descent (SGD) with adversaries, with various works such as \cite{chen2017distributed, mhamdi18a, yin2018byzantine, meamed, alistarh2018byzantine, diakonikolas2019sever, allen2020byzantine, prasad2020robust}. Most studies typically employ an aggregator to merge gradient estimates from workers while filtering out unreliable ones  \cite{meamed, chen2017distributed, mhamdi18a, yin2018byzantine}. However, many of these approaches make stringent assumptions regarding gradient noise, making it challenging to compare different aggregators. Notably, \cite{farhadkhani22a} introduces the concept of $(f,\lambda)$-resilient averaging, demonstrating that several well-known aggregators can be seen as special instances of this idea. In our work, we explore $(f,\lambda)$-averaging for combining policy gradient estimates from workers.

{4. \bf Byzantine Fault Tolerance in Distributed Reinforcement Learning:}. We note that \cite{fan2021fault} introduces an SVRG-PG based algorithm with a local sample complexity bound of $\cO\left(\frac{1}{N^{2/3} \varepsilon^{5/3}}+\frac{\alpha^{4/3}}{\varepsilon^{5/3}}\right)$, where $N$ represents the number of workers and $\alpha$ denotes the fraction of adversarial workers. However, they do not provide global convergence guarantees and require that the variance of importance sampling weights is upper bounded, which is unverifiable in practice. Moreover, their result is sub-optimal in $N$, which fails to provide linear speedup even when no adversaries are present. They additionally perform variance reduction with samples drawn from the central server, assuming its reliability. This is not generally allowed under federated learning since this could lead to data leakage and privacy issues \cite{FL-survey}. In contrast, our study centers on the global convergence of federated RL while keeping the server process minimal,  aggregating different gradients using  $(f,\lambda)$-resilient aggregator \cite{farhadkhani22a}, without   requiring additional samples at server.

%\cite{Chen} provides online and offline Byzantine-robust RL algorithms with uneven batch sizes. In particular,  they obtain a regret bound of order $\mathcal{O}( \sqrt{NT} + \alpha N \sqrt{T})$ for their online algorithm. However, their results only hold under the tabular setting, while our work focuses on general parametrization. 

Distributed RL algorithms with adversaries have been studied for episodic tabular MDPs in \cite{Chen}. Another line of works focus on empirical evaluations of Federated RL with adversaries \cite{lin2022byzantine,rjoub2022explainable,zhang2022resilient,xu2021signguard} and do not provide sample complexity guarantees.

\textbf{Main Contributions:} In this paper, we address the following fundamental question:

{\em What is the influence of adversaries on the global convergence sample complexity of Federated Reinforcement Learning (FRL)?}

To tackle this question, we introduce Resilient Normalized Hessian Aided Recursive Policy Gradient (Res-NHARPG). \textcolor{black}{Res-NHARPG integrates resilient averaging with variance-reduced policy gradient. Resilient averaging combines gradient estimates in a manner that minimizes the impact of adversaries on the algorithm's performance, while variance reduction accelerates convergence.} Our main result is that the global sample complexities of Res-NHARPG with $(f,\lambda)$ averaging is $\cO \left( \frac{1}{\epsilon^2} \log\left(\frac{1}{\epsilon}\right) \left( \frac{1}{N} + \lambda^2 \log(N)\right)\right)$. Here, $N$ denotes the number of workers and $f<N/2$ denotes the number of faulty workers. %\gt{Mention what $\epsilon$ also is}. 
The significance of our contributions can be summarized as follows.

\begin{enumerate}
    \item This paper provides the first global convergence sample complexity findings for federated policy gradient-based approaches with general parametrization \textcolor{black}{in the presence of adversaries}. % We emphasize that our proposed algorithm does not require the server to sample trajectories.

    \item We derive the sample complexity for Res-NHARPG for a broad class of aggregators, called $(f,\lambda)$-aggregators \cite{farhadkhani22a}. This includes several popular methods such as Krum \cite{krum}, Co-ordinate Wise Median (CWMed) and Co-ordinate Wise Trimmed Mean (CWTM) \cite{yin2018byzantine}, Minimum Diameter Averaging (MDA) \cite{mhamdi18a}, Geometric Median \cite{chen2017distributed} and Mean around Medians (MeaMed) \cite{meamed}. 
    
    \item We observe that for certain choices of aggregators (MDA, CWTM, MeaMed), our proposed approach achieves the optimal sample complexity of  $\Tilde{\cO} \left( \frac{1}{N\epsilon^2} \left(1 + \frac{f^2}{N}\right)\right)$, where $\Tilde{\cO}$ ignores logarithmic factors (see Remark \ref{rem:optimal-bound}). In particular, this implies linear speedup when $f = \cO(N^{\delta})$, where $\delta \leq 0.5$.
\end{enumerate}

We also provide experimental results showing the effectiveness of our algorithm under popular adversarial attacks (random noise, random action and sign flipping) under different environments (Cartpole-v1 from OpenAI Gym and InvertedPendulum-v2, HalfCheetah, Hopper, Inverted Double Pendulum and Walker from MuJoCo).  

%Our proof technique focuses on bounding the deviation between the update direction and the true gradient, since this is the main bottleneck in analysis of Byzantine models. We obtain sharp bounds by showing that the update direction is a bounded martingale difference,  and utilize Azuma-Hoeffding type bounds. 

We obtain a significant improvement over the original bounds for $(f,\lambda)$-resilient averaging in the context of stochastic optimization \cite{farhadkhani22a}. We achieve this by showing that our policy gradient update direction has certain desirable properties that allows us to utilize sharper concentration inequalities, instead of a simple union-bound (see Section \ref{sec:proof-outline}). 

\section{Problem Setup}
\label{sec:background}

%{\bf Markov Decision Process (MDP):} 

We consider a discounted Markov decision process  defined by the  tuple $(\cS,\cA,\cP,\textcolor{black}{r},\gamma)$, where $\cS$ denotes the state space, $\cA$ denotes the action space, $\cP(s'| s,a)$ is the probability of transitioning from state $s$ to state ${s}'$ after taking action $a$, $\gamma\in(0,1)$ is the discount factor, and $r:\cS\times\cA\to [-R,R]$ is the reward  function of $s$ and $a$. At each time $t$, the agent is at the current state $s_t\in\cS$ and takes action $a_t\in\cA$ based on a possibly stochastic policy $\pi:\cS\to \cP(\cA)$, i.e., $a_t\sim \pi(\cdot| s_t)$. 
The agent then obtains a reward $r_t=r(s_t,a_t)$. The sequence of state-action pairs $ \tau = \{s_0,a_0,s_1,a_1,s_2,a_2,\cdots\}$ is called a trajectory.
% The value function of a policy $\pi$, $V^{\pi}(s)$, is defined as follows 
% \begin{equation}   
% V^{\pi}(s)\coloneqq \E_{a_t\sim \pi(\cdot| s_t),s_{t+1}\sim \cP(\cdot| s_t,a_t)}\left[\sum_{t=0}^\infty \gamma^t r_t  \mathrel{\bigg|} s_0=s\right].\nonumber
% \end{equation}

We consider a parameter-server architecture with a trusted central server and $N$ distributed workers or agents, labelled $1, 2, \cdots, N$. We assume each agent is given an independent and identical copy of the MDP. Among these $N$ agents, we assume that $f$ are adversarial, with these agents providing any arbitrary information {\color{black} (such adversaries are also called Byzantine adversaries)}. The aim is for the agents to collaborate using a federated policy gradient based approach to come up with a policy $\pi$ that maximizes the value function, i.e.,
\begin{align}  
\label{eq:max-j}
\max_{\pi}~~J(\pi)\coloneqq\E_{s_0 \sim \rho, a_t\sim \pi(\cdot| s_t),s_{t+1}\sim \cP(\cdot| s_t,a_t)}\left[\sum_{t=0}^\infty \gamma^t r_t  \right], 
\end{align}
where the initial state $s_0$ is drawn from some distribution $\rho$. In practice, the state and action spaces are typically very large and thus the policy is parameterized by some $\theta \in \Rd$. The problem in \eqref{eq:max-j} now becomes finding $\theta$ using a collaborative approach among agents that maximizes $J(\pi_\theta)$. However since this maximization problem is usually nonconvex, it is difficult to find the globally optimal policy.

%{\bf Policy Gradient (PG):} 
\if 0
Let the initial state $s_0$ be drawn from some distribution $\rho$. Then,  the goal of the agent is to find the optimal policy that maximizes the value function, i.e.,
\begin{align}  
\label{eq:max-j}
\max_{\pi}~~J(\pi)\coloneqq\E_{s_0\sim\rho}[V^\pi(s_0)].
\end{align}
\fi

Let $J(\theta)$ denote $J(\pi_{\theta})$ henceforth and let $J^* = \max_{\pi} J(\pi)$. %The ``richness" of the policy parametrization dictates the gap between $J^*$ and $\max_{\theta} J(\theta)$. 
% Thus we can rewrite the problem as finding  $\theta$ such that
% \begin{align}
% \label{eq: global-conv-defn}	
% J^{\star} -  J(\theta) \leq \cO(\sqrt{\varepsilon_{\text{bias}}})+\varepsilon,
% \end{align}
% where the $\cO(\sqrt{\varepsilon_{\text{bias}}})$  term arises from limited expressive power of the policy parametrization $\pi_{\theta}$ (see Assumption \ref{assump: compatible error}). 
The optimization problem \eqref{eq:max-j} can be solved using the Policy Gradient (PG) approach at each agent, where the gradient $\nabla J(\theta)$ is estimated through sampling trajectories at that agent and observing rewards collected and then used in gradient ascent. 

Using the current estimate of the parameter $\theta$, we aim to find $\nabla J(\theta)$ using the sampled trajectories at the agent. Let $\cG \subset \{1,2,\cdots,N\}$ be the set of good workers of size $N-f$. For $n \in \cG$, let $\tau^{(n)}=\{s^{(n)}_0,a^{(n)}_0,s^{(n)}_1,\cdots,a^{(n)}_{H-1},s^{(n)}_H\}$ be a trajectory sampled by agent $n$ under policy $\pi_\theta$ of length $H$ and $\cR(\tau^{(n)}) = \sum_{h=0}^H \gamma^h r(s_h^{(n)},a_h^{(n)})$ be the sample return. We denote the distribution of a trajectory $\tau$ of length $H$ induced by policy $\pi_{\theta}$ with initial state distribution $\rho$ as $p^H_{\rho}(\tau|\theta)$, which can be expressed as  $p_{\rho}^H(\tau| \theta) = \rho(s_0) \prod_{h=0}^{H-1} \pi_{\theta} (a_h|s_h)\cP(s_{h+1}|s_h,a_h).$ The gradient of the truncated expected return function $J_H(\theta)\coloneqq\E_{s_0 \sim \rho, a_t\sim \pi_{\theta}(\cdot| s_t),s_{t+1}\sim \cP(\cdot| s_t,a_t)}[\cR(\tau)]$ can then be written as
\begin{align*}
    \nabla J_H(\theta) &= \int_{\tau} \cR(\tau) \nabla p^H_{\rho}(\tau|\theta) d\tau= \int_{\tau} \cR(\tau) \frac{\nabla p^H_{\rho}(\tau|\theta)}{p^H_{\rho}(\tau|\theta)}p^H_{\rho}(\tau|\theta) d\tau= \E [\cR(\tau) \nabla  \log p^H_{\rho}(\tau|\theta)].
\end{align*}
 REINFORCE \cite{williams1992simple} and GPOMDP \cite{baxter2001infinite} are commonly used estimators for the gradient $\nabla J(\theta)$. In this work, we will be using the GPOMDP estimator given below:
\begin{align}
\label{eq:trunc-GPOMDP}
	g(\tau^{(n)}, \theta)= \sum_{h=0}^{H} \left(\sum_{t=0}^{h} \nabla_{\theta} \log \pi_{\theta}(a^{(n)}_t | s^{(n)}_t)\right) \gamma^h r(s^{(n)}_h, a^{(n)}_h).   
\end{align}
The above expression is an unbiased stochastic estimate of $\nabla J_H(\theta)$ \cite{baxter2001infinite}.

\section{Proposed Algorithm}

%Each good agent $n$ needs to estimate $\nabla J(\theta)$ based on their trajectories.

In Res-NHARPG (given in Algorithm \ref{alg:(N)-HARPG}), each good agent $n \in \cG$
computes a variance-reduced gradient estimator, $d^{(n)}_t$, in a recursive manner as follows:
\begin{align}
\label{eq:dt_iterate}
    d_t^{(n)} =& (1 - \eta_t) ( d_{t-1}^{(n)} +  B(\hat{\tau}_t^{(n)} ,\hat{\theta}_t^{(n)})(\theta_t - \theta_{t-1}) ) + \eta_t g(\tau_t^{(n)}, \theta_t), 
\end{align}
where $\{\eta_t\}_t$ denotes a suitable choice of parameters,  $B(\tau , \theta) \coloneqq \nabla \Phi(\tau,\theta) \nabla \log p(\tau|\pi_\theta)^T + \nabla^2 \Phi (\tau,\theta)$
with $\Phi(\tau,\theta) \coloneqq \sum_{t=0}^{H-1} \left(\sum_{h=t}^{H-1}\gamma^h r(s_h,a_h)\right)\log \pi_{\theta}(a_t,s_t)$, $\tau_t^{(n)} \sim p_{\rho}^H(\cdot|\pi_{\theta_t}), \, \hat{\tau}_t^{(n)} \sim p_{\rho}^H(\cdot|\pi_{\hat{\theta}_{t}^{(n)}})$, {\color{black} and $\hat{\theta}_{t}^{(n)} = q_t^{(n)} \theta_{t} + (1-q_t^{(n)}) \theta_{t-1}$ where $ q_t^{(n)}$ is sampled from $\mathcal{U}([0,1])$. }

\begin{wrapfigure}{r}{.5\textwidth}
	\vspace{-.4in}
	\begin{minipage}{.5\textwidth}
		\begin{algorithm}[H]
			\caption{Resilient Normalized Hessian-Aided Recursive Policy Gradient (Res-NHARPG)}\label{alg:(N)-HARPG}
			\begin{algorithmic}[1]
				\State \textbf{Input}: $\theta_0$, $\theta_1$, $d_0$, $T$, $\{\eta_t\}_{t\geq 1}$, $\{\gamma_t\}_{t\geq 1}$

				\For{$t=1, \ldots, T -  1$}
				
				\LineComment{Server broadcasts $\theta_{t}$ to all agents}
				\LineComment{Agent update}
				\For{each agent $n \in [N]$}{ in parallel}  
				\State $q_t^{(n)} \sim \mathcal{U}([0,1])$
				\State $\hat{\theta}_{t}^{(n)} = q_t^{(n)} \theta_{t} + (1-q_t^{(n)}) \theta_{t-1}$
				\State $\tau_t^{(n)} \sim p_{\rho}^H(\cdot|\pi_{\theta_t}); \, \hat{\tau}_t^{(n)} \sim p(\cdot|\pi_{\hat{\theta}_{t}^{(n)}})$
				\State \label{n-harpg:hvp} $v_t^{(n)} = B(\hat{\tau}_t^{(n)}, \hat{\theta}_{t}^{(n)})(\theta_t - \theta_{t-1})$
				\State $d_t^{(n)} = (1 - \eta_t) ( d_{t-1}^{(n)} +  v_t^{(n)} ) + \eta_t g(\tau_t^{(n)}, \theta_t) $
				\EndFor
				\LineComment{Server Update}
				\State $d_t = F(d_t^{(1)},\cdots,d_t^{(N)})$
				\State $\theta_{t+1} = 
				\theta_t + \gamma_t \frac{d_t}{\|d_t\|}$
				
				\EndFor
				\State \Return{$\theta_T$}
			\end{algorithmic}
		\end{algorithm}
	\end{minipage}
	\vspace{-.1in}
\end{wrapfigure}

The update in \eqref{eq:dt_iterate} is inspired by the N-HARPG algorithm in \cite{fatkhullin} since its use of Hessian information was shown to help achieve sample complexity of order $\cO\left(\frac{1}{\epsilon^2}\right)$, which is currently the state of the art. \textcolor{black}{The algorithm draws inspiration from the STORM variance reduction technique \cite{cutkosky-orabona19}, but instead of relying on the difference between successive stochastic gradients, it incorporates second-order information \cite{BetterSGDUsingSOM_Tran_2021}. This approach removes the dependency on Importance Sampling (IS) and circumvents the need for unverifiable assumptions to bound the IS weights \cite{fatkhullin}.  In Step 10 of Algorithm \ref{alg:(N)-HARPG}, the update direction $d_t^{(n)}$ is computed by adding a second-order correction, $(1-\eta_t)v_t^{(n)}$ (as defined in Step 9), to the momentum stochastic gradient $(1-\eta_t)d^{(n)}_{t-1} + \eta_t g(\tau_t^{(n)}, \theta_t)$. The uniform sampling procedure in Steps 6-9 ensures that $v_t^{(n)}$ is an unbiased estimator of $\nabla J_H(\theta_t) - \nabla J_H(\theta_{t-1})$, closely resembling the term used in the original STORM method \cite{cutkosky-orabona19}.
%In \cite{cutkosky-orabona19}, the update direction $d_t$ is calculated by applying a correction term $(1-\eta_t)v_t^{(n)}$, where $v_t^{(n)}$ represents a stochastic estimate of $\nabla J(\theta_t) - \nabla J(\theta_{t-1})$. In our case, as seen in \eqref{eq:dt_iterate}, we define $v_t^{(n)} = B(\tau_t^{(n)}, \hat{\theta}_t^{(n)})(\theta_t - \theta_{t-1})$. %This makes $v_t^{(n)}$ a stochastic estimate of $\nabla^2 J(\hat{\theta}_t^{(n)})(\theta_t - \theta_{t-1})$.  This choice ensures that $v_t^{(n)}$ remains an unbiased estimator of $\nabla J(\theta_t) - \nabla J(\theta_{t-1})$.
}

%$B(\tau,\theta)$ is an unbiased stochastic estimator of $\nabla^2 J_H(\theta)$ \cite{shen19d}. 
In Step 9, we do not need to compute and store $B(\tau,\theta)$ but only a term of form $ B(\tau,\theta)u$ which can be easily computed via automatic
differentiation of the scalar quantity $\langle g(\tau,\theta),u \rangle$ \cite{fatkhullin}. This allows us to exploit curvature information from the policy Hessian without compromising the per-iteration computation cost.

{\bf Aggregation at the Server: } At each iteration $t$, a good agent always sends its computed estimate back to the server, while an adversarial agent may return any arbitrary vector. At the server, we aim to mitigate the effect of adversaries is by using an aggregator to combine estimates from all agents such that bad estimates are filtered out.

%samples $M$ trajectories and calculates an estimate of $\nabla J(\theta_k)$ as follows
% \begin{align}
% \label{eq:multi-gpomdp}
%   u^{(n)}_k =  \frac{1}{M}\sum_{i=1}^M g(\tau_i^{(n)} | \theta_k),
% \end{align}
% which is sent back to the server. 

The server receives $d_t^{(n)}$, $n=1,2,\cdots,N$ and uses the following update
\begin{align}
\label{eq: aggregate update}
d_t = F(d^{(1)}_t,d^{(2)}_t,\cdots,d^{(N)}_t), \text{ and}
\end{align}
% The stochastic PG ascent update is given below 
\begin{align}
\label{eq: policy para update}
\theta_{k+1} = \theta_{k} + \gamma_t \frac{d_t}{\|d_t\|},
\end{align}
where $F$ is an aggregator and $\gamma_t>0$ is the stepsize. {\color{black}For the robust aggregator $F$, we consider $(f,\lambda)$-averaging introduced in \cite{farhadkhani22a} as this family of aggregators encompasses several popularly used methods in Byzantine literature. The aggregator function aims to make our algorithm resilient to adversaries. The last step of the algorithm (Step 14) uses the direction $d_t$ with normalization to update the parameter $\theta_{t+1}$.}

We now define $(f,\lambda)$-resilient averaging as in \cite{farhadkhani22a}  below for ease of reference, where $f<N/2$. We note that $f<N/2$ is the optimal breakdown point as no aggregator can possibly tolerate $f \geq N/2$ adversaries \cite{karimireddy2020byzantine}.  

\begin{definition}[{\bf $(f, \, \lambda)$-Resilient averaging}]
\label{def:resilient-aggregator}
For $f < N/2$ and real value $\lambda \geq 0$, an aggregation rule $F$ is called {\em $(f, \, \lambda)$-resilient averaging} if for any collection of $N$ vectors $x_1, \ldots, \, x_N$, and any set $\cG \subseteq \{1, \ldots, \, N\}$ of size $N-f$,
\begin{align*}
    \| F(x_1, \ldots, \, x_N) - \overline{x}_{\cG} \| \leq \lambda \max_{i, j \in \cG} \| x_i - x_j \|
\end{align*}
where $\overline{x}_{\cG} \coloneqq \frac{1}{|\cG|} \sum_{i \in \cG} x_i$, and $|\cG|$ is the cardinality of $\cG$.
\end{definition}

%{\bf \color{red} Add a note that for $f\ge N/2$, adversaries cannot be handled. }

% \begin{table*}[t]
% \begin{center}
% \begin{tabular}{cc}
% \hline
% \textbf{Aggregator}  & \textbf{\thead{Sample Complexity \\ of Res-NHARPG}} \\
% \hline 
% MDA         & $\cO \left( \frac{1}{\epsilon^2} \log\left(\frac{1}{\epsilon}\right) \left( \frac{1}{N-f} + \frac{f^2 \log(N-f)}{(N-f)^2}\right)\right)$ \\
% CWTM             &$\cO \left( \frac{1}{\epsilon^2} \log\left(\frac{1}{\epsilon}\right) \left( \frac{1}{N-f} + \frac{f^2 \Delta^2  \log(N-f)}{(N-f)^2}\right)\right)$ \\
% CWMed             & $\cO \left( \frac{1}{\epsilon^2} \log\left(\frac{1}{\epsilon}\right) \left( \frac{1}{N-f} + \frac{N^2 \Delta^2  \log(N-f)}{(N-f)^2}\right)\right)$\\
% Krum         & $\cO \left( \frac{1}{\epsilon^2} \log\left(\frac{1}{\epsilon}\right) \left( \frac{1}{N-f} + \left(1+\sqrt{\frac{N-f}{N-2f}}\right)^2\log(N-f)\right)\right)$ \\
% MeaMed         & $\cO \left( \frac{1}{\epsilon^2} \log\left(\frac{1}{\epsilon}\right) \left( \frac{1}{N-f} + \frac{f^2 \Delta^2  \log(N-f)}{(N-f)^2}\right)\right)$  \\
% GM             & $\cO \left( \frac{1}{\epsilon^2} \log\left(\frac{1}{\epsilon}\right) \left( \frac{1}{N-f} + \left(1+\sqrt{\frac{(N-f)^2}{N(N-2f)}}\right)^2\log(N-f)\right)\right)$ \\
% \hline
% \end{tabular}
% \end{center}
% \caption{In the above table, $\Delta=\min \{\sqrt{d},2\sqrt{N-f}\}$.} 
% \label{table:sample complexities}
% \end{table*}

Several well-known aggregators such as Krum, CWMed, CWTM, MDA, GM and MeaMed are known to be $(f,\lambda)$-resilient. {\color{black} We note that for these aggregators, $\lambda$ as a function of $f$ is summarized in Table \ref{table:sample complexities}, where these values are from \cite{farhadkhani22a}. } As a result, our unified analysis provides guarantees for all the above mentioned methods. We define and discuss these aggregators in Appendix \ref{sec:details-aggregators}.

%the MDA function below while defining other aggregator functions studied in Appendix A. {\bf \color{red} Add that complexities also in Appendix. }

 %Let $[x]_k$ denote the $k$-th coordinate of $x \in \Rd$. Given a set of $n$ vectors $X=\{x_1, \ldots, \, x_n\}$ as input, the output of different aggregator functions is given below.

%\gt{What does `each algorithm' mean? Perhaps, you meant ``we now discuss four standard aggregators and describe their output. It may be a good idea to keep one definition here as an example and move the rest to the appendix, since there is nothing novel in here. The focus of this section should be on the contributions and not what is there in the literature.}

%

%  

% median-based aggregation
% rules, without any averaging operation. Thus, the variance
% of their outputs grows with n, as suggested by the standard
% bounds from order statistics (Arnold & Groeneveld, 1979;
% Bertsimas et al., 2006). On the contrary, MDA, CWTM,
% and MeaMed perform an averaging operation after filtering
% out dubious vectors, thus mimicking the variance reduction
% property of the averaging scheme traditionally used in the
% vanilla distributed SGD.

% role of variance reduction?? PG and SGD alone cant distinguish. SRVR-PG vs SGD is possible.

\section{Assumptions and Main Result}

In this section,  we first introduce  a few notations and list the assumptions used for our results.  Let $d^{\pi_\theta}_{\rho}\in\cP(\cS)$ denote the state visitation measure induced by policy $\pi_\theta$ and initial distribution $\rho$ defined as
\begin{align}
    d^{\pi_\theta}_{\rho}(s) \coloneqq (1-\gamma)\E_{s_0\sim\rho}\sum_{t=0}^\infty\gamma^t \PP(s_t=s| s_0,\pi_\theta)
\end{align}
and $\nu^{\pi_{\theta}}_{\rho}(s,a) \coloneqq d^{\pi_\theta}_{\rho}(s)\pi(a | s)$ be the state-action visitation measure induced by $\pi_\theta$. % and initial state distribution $\rho$. 

We assume that the policy parametrization $\pi_{\theta}$ is a good function approximator, which is measured by the \textit{transferred compatible function approximation error}. This assumption is commonly used in obtaining global bounds with parametrization \cite{agarwal2021theory,liu2020improved}.

\begin{assumption}
\label{assump: compatible error}
For any $\theta\in\Rd$, the \textit{transferred compatible function approximation error}, $L_{\nu^{\star}}(w^{\theta}_{\star}; \theta)$, satisfies 
\begin{align}
\label{equ: minimal compatible function approximation error}
&L_{\nu^{\star}}(w^{\theta}_{\star}; \theta) \coloneqq \E_{(s,a)\sim \nu^{\star}}\left[\big(A^{\pi_{\theta}}(s,a)-(1-\gamma)(w^{\theta}_{\star})^\top\nabla_{\theta}\log\pi_{\theta}(a| s)\big)^2\right]\leq \varepsilon_{\mathrm{bias}},
\end{align} 
where \textcolor{black}{$A^{\pi_{\theta}}(s,a)$ is the advantage function of policy $\pi_{\theta}$ at $(s,a)$}, $\nu^{\star}(s,a) = d^{\pi^{\star}}_{\rho}(s) \cdot \pi^{\star}(a | s)$ is the state-action distribution induced by an optimal policy $\pi^{\star}$ that maximizes $J(\pi)$, and $w^{\theta}_{\star}$ is the exact Natural Policy Gradient update direction at $\theta$.
\end{assumption}
$\varepsilon_{\mathrm{bias}}$ captures the parametrization capacity of $\pi_{\theta}$. For $\pi_{\theta}$ using the softmax parametrization, we have $\varepsilon_{\mathrm{bias}}=0$ \cite{agarwal2021theory}. %\gt{Don't you want $\theta$ to be of size $|\cS| \times |\cA|?$}
When $\pi_{\theta}$ is a restricted parametrization, which may not contain all stochastic policies, we have $\varepsilon_{\mathrm{bias}}>0$. It is known that $\varepsilon_{\mathrm{bias}}$ is very small when rich neural parametrizations are used \cite{wang2019neural}. 

%%%%%%%% yanli used this assumption for global convergence result
\begin{assumption}
\label{assump: strong convexity}
For all $\theta\in\Rd$, 
% and $s\in\cS$, 
 the Fisher information matrix induced by policy $\pi_{\theta}$ and initial  state distribution $\rho$ satisfies
\begin{align*}
 F_{\rho}(\theta) &\coloneqq\E_{(s,a)\sim \nu^{\pi_{\theta}}_{\rho}}\left[\nabla_{\theta}\log\pi_{\theta}(a| s)\nabla_{\theta}\log\pi_{\theta}(a| s)^\top \right]\succcurlyeq \mu_F \cdot I_d,
\end{align*}
for some constant $\mu_F>0$. 
\end{assumption}

\begin{assumption}
\label{assump: variance}
    There exists $\sigma>0$ such that $g(\tau,\theta)$ defined in \eqref{eq:trunc-GPOMDP} satisfies
    $\E \|g(\tau,\theta)-\E [g(\tau,\theta)] \|^2 \leq \sigma^2,$
    for all $\theta$ and $\tau \sim p_{\rho}^H(\cdot|\theta)$.
\end{assumption}

\begin{assumption}
\label{assump: conditions on score function}
\begin{enumerate}

    \item $\|\nabla_{\theta}\log \pi_{\theta}(a| s)\|\leq G_1$ for any $\theta$ and $(s,a)\in\cS\times \cA$.
    \item $\|\nabla_{\theta}\log \pi_{\theta_1}(a| s)-\nabla_{\theta}\log \pi_{\theta_2}(a| s)\|\leq G_2 \|\theta_1-\theta_2\|$ for any $\theta_1, \theta_2$ and $(s,a)\in\cS\times \cA$.
\end{enumerate}
\end{assumption}
%\textcolor{red}{\bf Are these comments fine?}    
{\bf Comments on Assumptions \ref{assump: strong convexity}-\ref{assump: conditions on score function}:}  We would like to highlight that all the assumptions used in this work are commonly found in PG literature. We elaborate more on these assumptions below. 
    
    Assumption \ref{assump: strong convexity} requires that the eigenvalues of the Fisher information matrix can be bounded from below. This assumption is also known as the \textit{Fisher non-degenerate policy} assumption and is commonly used in obtaining global complexity bounds for PG based methods \cite{liu2020improved,zhang2021on,Bai_Bedi_Agarwal_Koppel_Aggarwal_2022,fatkhullin}. 
    
    Assumption \ref{assump: variance} requires that the variance of the PG estimator must be bounded and \ref{assump: conditions on score function} requires that the score function is bounded and Lipschitz continuous. Both assumptions are widely used in the analysis of PG based methods \cite{liu2020improved,agarwal2021theory, papini2018stochastic, xu2019improved,xu2019sample,fatkhullin}.
    
    Assumptions \ref{assump: strong convexity}-\ref{assump: conditions on score function} were shown to hold for various examples recently including Gaussian policies with linearly parameterized means \textcolor{black}{with clipping} \cite{liu2020improved, fatkhullin}.
    
   % In particular, to ensure Assumption \ref{assump: conditions on score function} for Gaussian policies, it is sufficient to bound the sampled actions and on the mean parameterization . In practice, this is usually enforced by clipping the actions selected by the policy and projecting the update of the gradient method into the ball of fixed radius.

\begin{table*}[t]
\resizebox{\textwidth}{!}{
%\begin{center}
\begin{tabular}{|c|c|c|c|c|}
\hline
\textbf{Aggregator} &\textbf{Computational}  & \textbf{$\lambda$} & \textbf{Sample Complexity of Res-NHARPG} & \textbf{Order} \\
 & \textbf{Complexity} &  &  &  \textbf{Optimal?}\\
\hline 
MDA   & NP-Hard & $\frac{2f}{N-f}$     & $\cO \left( \frac{1}{\epsilon^2} \log\left(\frac{1}{\epsilon}\right) \left( \frac{1}{N} + \frac{f^2 \log(N)}{N^2}\right)\right)$  & Yes\\
\hline
CWTM   & $\Theta(dN)$ &  $\frac{f}{N-f}  \Delta$          &$\cO \left( \frac{1}{\epsilon^2} \log\left(\frac{1}{\epsilon}\right) \left( \frac{1}{N} + \frac{f^2 \Delta^2  \log(N)}{N^2}\right)\right)$ & Yes \\
\hline
MeaMed & $\Theta(dN)$ & $\frac{2 f}{N-f} \Delta$        & $\cO \left( \frac{1}{\epsilon^2} \log\left(\frac{1}{\epsilon}\right) \left( \frac{1}{N} + \frac{f^2 \Delta^2  \log(N)}{N^2}\right)\right)$  & Yes\\
\hline
CWMed  & $\Theta(dN)$ & $\frac{N}{2(N-f)}  \Delta$          & $\cO \left( \frac{1}{\epsilon^2} \log\left(\frac{1}{\epsilon}\right) \left( \frac{1}{N} +  \Delta^2  \log(N)\right)\right)$ & No\\
\hline
Krum & $\Theta(dN^2)$ & $1+\sqrt{\frac{N-f}{N-2f}}$       & $\cO \left( \frac{1}{\epsilon^2} \log\left(\frac{1}{\epsilon}\right) \left( \frac{1}{N} + \frac{N-f}{N-2f}\log(N)\right)\right)$ & No\\
\hline
GM & $\cO(dN\log ^3(N/\epsilon))$ $\color{blue} ^{(1)}$ & $1+\sqrt{\frac{(N-f)^2}{N(N-2f)}}$           & $\cO \left( \frac{1}{\epsilon^2} \log\left(\frac{1}{\epsilon}\right) \left( \frac{1}{N} + \frac{N-f}{N-2f}\log(N)\right)\right)$   & No\\
\hline
\end{tabular}
%\end{center}
}
\vspace{-.1in}
\caption{\small In the above table, $\Delta=\min \{\sqrt{d},2\sqrt{N-f}\}$. The computational complexity of the aggregator, and the order optimality of sample complexity (in terms of $f$, $N$, and $\epsilon$) are also mentioned in the Table. Remark \ref{rem:comment_optimal} provides a discussion on the sample complexity bounds. $\color{blue} ^{(1)}$ We note that while GM is convex optimization with no closed-form solution, an $(1+\epsilon)$-approximate solution can be found in $\cO(dN\log ^3(N/\epsilon))$ time \cite{GM-computation}.}
\label{table:sample complexities}
\vspace{-.1in}
\end{table*}
We now state the main result which gives us the \textcolor{black}{last iterate} global convergence rate of Algorithm \ref{alg:(N)-HARPG}:
\begin{theorem}
\label{thm:Res-NHARPG}
    Consider Algorithm \ref{alg:(N)-HARPG} with $\gamma_t = \frac{6G_1}{\mu_F(t+2)}$, $\eta_t = \frac{1}{t }$ and $H = (1-\gamma)^{-1}\log (T+1)$. Let Assumptions \ref{assump: compatible error}, \ref{assump: strong convexity}, \ref{assump: variance} and \ref{assump: conditions on score function} hold.  Then for every $T \geq 1$ the output $\theta_T$ satisfies
    \begin{align*}
    J^* - J(\theta_{T})  = 
    \frac{\sqrt{\varepsilon_{\mathrm{bias}}}}{1-\gamma}
    &+ \mathcal{O} \left(\sqrt{\frac{\log T}{NT}} + \lambda \sqrt{\frac{\log N + \log T}{T}} \right).
    \end{align*}
\end{theorem}

From Theorem \ref{thm:Res-NHARPG}, the number of trajectories required by Algorithm \ref{alg:(N)-HARPG} to ensure 
   $J^* - J(\theta_{T}) \leq \frac{\sqrt{\varepsilon_{\mathrm{bias}}}}{1-\gamma} + \epsilon$    is     
   \begin{align*}
    \cO \left( \frac{1}{\epsilon^2} \log\left(\frac{1}{\epsilon}\right) \left( \frac{1}{N} + \lambda^2 \log(N)\right)\right). 
    \end{align*}
For ease of exposition, we keep only dependence on $\epsilon$, $N$ and $\lambda$ while  providing a detailed expression in \eqref{eq:final-bound-all-terms}. {\color{black}We also note that the expression implicitly reflects the dependence on $f$, as the value of $\lambda$ depends on both $f$ and the choice of the aggregator.} The sample complexity for different aggregator functions are summarized in Table \ref{table:sample complexities}. The value of $\lambda$ in the table for each aggregator is from \cite{farhadkhani22a}.

%{\bf \color{red} Add detailed expression as in Rebuttal with a line on $T$. }
%   {\bf \color{red}  Mention that for MDA, CWTM and MeaMed when number of adversarial agents $f=N^\delta$, and $\delta<1/2$ is linear speedup and for $1/2<\delta<1$, is speedup. }

%\textcolor{red}{\bf Can you check if the remarks are ok?}

%{\bf Remarks:}
%\begin{remark}[leftmargin=*] 
    %\item Consider the case where there is only one worker and no adversary, i.e., $N=1$ and $f=0$. Then it can be seen that the sample complexities obtained for Res-PG and Res-SRVR-PG are of the same order as sample complexities of PG and SRVR-PG in \cite{liu2020improved}. 
    
%    \item (see Table \ref{table:sample complexities}). %This may be due to the fact that MDA, CWTM and MeaMed average the estimates after filtering out the bad ones, therby reducing the variance. In contrast, the bounds for CWMed, Krum and GM are worse since they do not perform any averaging. 

   % \gt{It is unclear if the above statements about CWMed, Krum, and GM can be made using just the upper bounds. Perhaps, you can add some simulations to support your claim.}
    
    %\item

{\color{black} We now remark on the lower bound  for the sample complexity. In order to do that, we will first provide the lower bound for  Stochastic Gradient Descent with Adversaries in \cite{alistarh2018byzantine}. 
\begin{lemma}[\citet{alistarh2018byzantine}]\label{lem:lb}
    For any $D$, $V$,  and $\epsilon>0$, let there exists a linear function $g : [-D, D] \to{\mathbb{R}}$ (of Lipschitz continuity $G = \epsilon/D$) with a stochastic estimator $g_s$ such that ${\mathbb E}[g_s]=g$ and $||\nabla g_s(x) - \nabla g(x)||\le V$ for all $x$ in the domain.  Then, given $N$ machines, of which $f$ are adversaries, and $T$ samples from the stochastic estimator per machine, no algorithm can output $x$ so that $g(x) - g(x^*) < \epsilon$ with probability  $\geq 2/3$
unless $T = {\tilde \Omega}\left(\frac{D^2V^2}{\epsilon^2 N} + \frac{f^2V^2D^2}{\epsilon^2N^2 }\right)$, where $x^* = \arg\min_{x\in [-D,D]} g(x)$.
\end{lemma}

     \begin{remark} 
     \label{rem:optimal-bound}
     Lemma \ref{lem:lb} provides the lower bound for the sample complexity of SGD with adversaries such as: $T = {\tilde \Omega} \left(\frac{1}{\epsilon^2} \left(\frac{1}{N} + \frac{f^2}{N^2}\right)\right).$ We note that the lower bound function class in Lemma \ref{lem:lb} may not satisfy the function class in this paper explicitly. However, we believe that since the lower bound holds for any class of functions that includes linear functions, the result should still hold with the assumptions in this paper. We also note that since $1/\epsilon^2$ is a lower bound for the centralized case \cite{mondal2023improved}, it follows that the lower bound for the distributed setup is ${\tilde \Omega}\left(\frac{1}{N\epsilon^2}\right)$ even in the absence of adversaries, which is the same as ${\tilde \Omega} \left(\frac{1}{\epsilon^2} \left(\frac{1}{N} + \frac{f^2}{N^2}\right)\right)$ when $f<\sqrt{N}$. 

    \end{remark}
}
    \begin{remark}
    \label{rem:comment_optimal}
    From Remark \ref{rem:optimal-bound}, it follows that MDA, CWTM and MeaMed achieve optimal sample complexity in terms of $\epsilon$, $N$ and $f$ (upto logarithmic factors). As a result, these methods exhibit linear speedup when the number of adversarial agents is on the order of $\mathcal{O}(\sqrt{N})$. When the number of adversarial agents is $\mathcal{O}(N^{\delta})$, where $\delta > 0.5$, they achieve speedup of order $\mathcal{O}(N^{2(1-\delta)})$. It's worth noting that while MDA poses computational challenges, CWTM and MeaMed are computationally inexpensive, on par with simple averaging used in vanilla federated policy gradient.
\end{remark}

\begin{remark}  
Works in Byzantine literature generally use strong assumptions on the noise. For example, there are a line of works assuming vanishing variance \cite{krum,mhamdi18a,meamed}, which does not hold even for the simplest policy parametrizations. Intuitively, when the assumptions on the noise is weakened, it becomes harder to distinguish adversarial and honest workers. In contrast, we provide these near-optimal bounds only using standard assumptions in Policy Gradient literature.
\end{remark}

\section{Proof Outline for Theorem  \ref{thm:Res-NHARPG}}
\label{sec:proof-outline}

% Let 
% \begin{align}
% \bar{d}_t \coloneqq \frac{1}{N-f} \sum_{i \in \cG} d^{(i)}_t.
% \end{align}

{\bf Key step:} The main challenge in analyses with  adversaries lies in bounding the difference of the update direction (in Algorithm 1, $d_t$) and the true gradient, $\nabla J_H (\theta_t)$. 

We bound $\E\| d_t - \nabla J_H(\theta_t) \| $ as
\begin{align*}
    \E\| d_t - \nabla J_H(\theta_t) \| \leq \E\| d_t - \bar{d}_{t} \| + \E\| \bar{d}_{t} - \nabla J_H(\theta_t) \|, 
\end{align*}

where $\bar{d}_t = \frac{1}{N-f} \sum_{n \in \cG} d_t^{(n)} $. We note that $\E\| d_t - \bar{d}_{t} \|$ can be bounded using the definition of $(f,\lambda)$-aggregators as follows (details given in Detailed Outline):
\begin{align}
      \E\| d_t - \bar{d}_{t} \| 
      \leq 2\lambda \E \Big[ \max_{i \in \cG} \| d_t^{(i)} - \nabla J_H(\theta_{t}) \|\Big]. 
\end{align}

In \cite{farhadkhani22a}, $\E[ \max_{i \in \cG} \| d_t^{(i)} - \nabla J_H(\theta_{t}) \|]$  is bounded by
\begin{align*}
    \E \Big[ \sum_{i \in \cG} \| d_t^{(i)} - \nabla J_H(\theta_{t}) \| \Big] = (N-f) \E\| \textcolor{black}{d_t^{(j)}} - \nabla J_H(\theta_{t}) \|,
\end{align*}
\textcolor{black}{ where $d_t^{(j)}$ is an unbiased estimate of $\nabla J_H(\theta_{t})$ obtained from an honest agent $j \in \cG$. This results in an extra factor of $(N-f)$.} To illustrate the impact of this term, consider the scenario where $\lambda$ is of optimal order, specifically $\mathcal{O}\left(\frac{f}{N-f}\right)$ and $f = \mathcal{O}(N^{\delta})$. Even with optimal $\lambda$ and with $0 < \delta \leq 0.5$, their sample complexity bounds do not achieve linear speedup. Furthermore, if the number of faulty workers $\delta > 0.5$, the convergence rate deteriorates with an increase in the number of workers $N$. More specifically, if $f = \Theta (N^{\delta+1/2})$, their sample complexity bound is $\mathcal{O}\left(\frac{N^{2\delta}}{\epsilon^2}\right)$, holding true for all $\delta>0$.
 
 In contrast, we show that for our algorithm, this term can be reduced to $\cO(\log (N-f))$, from $(N-f)$. This results in significant improvement in the sample complexity bound, enabling us to obtain linear speedup when $\delta \leq 0.5$ and sample complexity of order $\mathcal{O}\left(\frac{N^{2(\delta-1)}}{\epsilon^2}\right)$ when $\delta > 0.5$, with certain choice of aggregators, showing speedup of order $\cO(N^{2(1-\delta)})$.

%{\bf \color{red} I think this needs revision - we are saying extra factor is log(N-f) rather than N-f. What is overall instead of $\cO\left(\frac{N^{2\delta}}{\epsilon^2}\right)$?}. 
We do this by showing that the update direction calculated by each worker $i$ in Algorithm \ref{alg:(N)-HARPG}, $d_t^{(i)}$, is a sum of bounded martingale differences. As a result, we can then invoke Azuma-Hoeffding inequality to obtain sharper bounds. This ensures that our algorithm guarantees speedup as long as $f= O(N^\delta)$, where $\delta<1$.

%\subsection{Proof Outline} 

{\bf Detailed Outline: } We split the proof into three steps given below:

\textbf{1. Bounding  $\E\|d_t - \bar{d}_t \|$}: 

From the definition of $(f,\lambda)$-aggregator, we have $ \| d_t - \bar{d}_{t} \| \leq \lambda \max_{i, j \in \cG} \| d_t^{(i)} - d_t^{(j)} \| $. 
Thus,
\begin{align}
      \E\| d_t - \bar{d}_{t} \| &\leq \lambda \E[ \max_{i, j \in \cG} \|  d_t^{(i)} - d_t^{(j)} \|] \leq \lambda \E[ \max_{i, j \in \cG} (\| d_t^{(i)} - \nabla J_H(\theta_{t})\| + \| d_t^{(j)} - \nabla J_H(\theta_{t}) \|)] \nonumber \\
      &\leq 2\lambda \E[ \max_{i \in \cG} \| d_t^{(i)} - \nabla J_H\textcolor{black}{(}\theta_{t}) \|]. 
\end{align}

Let $X_i \coloneqq \| d_t^{(i)} - \nabla J_H(\theta_{t})\|$ and $X \coloneqq  \max_{i \in \cG} X_i$. We shall denote the indicator function of the event $A$ by $\mathbf{1}_{A}$. Then for all $\bar{\epsilon}>0$

\begin{align*}
    \E[X] &= \E[X \mathbf{1}_{\{X \geq \bar{\epsilon}\}} + X \mathbf{1}_{\{X < \bar{\epsilon}\}}]\leq \E[X \mathbf{1}_{\{X \geq \bar{\epsilon}\}} + \bar{\epsilon}] \leq \E[C_1 \mathbf{1}_{\{X \geq \bar{\epsilon}\}} + \bar{\epsilon}] \\
   & \leq C_1 \PP(X \geq \bar{\epsilon}) + \bar{\epsilon} 
    \leq C_1 \sum_{i \in \cG}\PP(X_i \geq \bar{\epsilon}) + \bar{\epsilon} = C_1 (N-f)\PP(\|d_t^{(i)} - \nabla J_H(\theta_t)\| \geq \bar{\epsilon}) + \bar{\epsilon},
\end{align*}

where $C_1$ is an upper bound on $X$ (see Appendix \ref{subsec:d_i_martingale}). Thus, for all $\bar{\epsilon} >0$
\begin{align}
\label{eq:f-lambda-bound-1}
    \E\| d_t - \bar{d}_{t} \| &\leq  2 C_1 \lambda (N-f)\PP(\|d_t^{(i)} - \nabla J_H(\theta_t)\| \geq \bar{\epsilon})+ 2\lambda \bar{\epsilon}. 
\end{align}
 In order to bound $\PP(\|d_t^{(i)} - \nabla J_H(\theta_t)\| \geq \bar{\epsilon})$, we make use of the following result. 
\begin{lemma}
\label{lem:d_i_martingale}
Consider Algorithm 1.  For all $i \in \cG$ and $t \geq 1$, we have
    \begin{align*}
        d_t^{(i)} - \nabla J_H(\theta_t) = \frac{1}{t} \sum_{j=1}^t M_j^{(i)},
    \end{align*}
    where $\|M_t^{(i)}\| \leq C_1$ and $ \E [M_{t+1}^{(i)} \mid M_{t}^{(i)}] = 0$.
\end{lemma}
The proof of the above lemma is given in Appendix \ref{subsec:d_i_martingale}. Using Vector Azuma-Hoeffding inequality, we obtain the following bound (proof in Appendix \ref{subsec:equation})
\begin{align}
\label{eq:dt-dt-bar-bound}
        \E\| d_t - \bar{d}_{t} \| &\leq  4e^2 C_1 \lambda (N-f) e^{-(t+1)\bar{\epsilon}^2/2C_1^2} + 2\lambda \bar{\epsilon}.
\end{align}

Let $\bar{\epsilon} = \sqrt{\frac{2C_1^2\log (N-f)(t+1)}{t+1}}$. Then,
\begin{align}
\label{eq:final-dt-dt-bar-bound}
        &\E\| d_t - \bar{d}_{t} \| \leq  4e^2 C_1 \lambda (N-f) e^{-(t+1)\bar{\epsilon}^2/2C_1^2} + 2\lambda \bar{\epsilon} = \frac{4e^2 C_1 \lambda}{t+1} + 2C_1\lambda \sqrt{\frac{2\log (N-f)(t+1)}{t+1}}.
\end{align}

\textbf{2. Bounding  $\E\| \bar{d}_t - \nabla J_H(\theta_t) \|$}: 

We provide a bound for $\E\| \bar{d}_t - \nabla J_H(\theta_t) \|$ below. The proof can be found in Appendix \ref{subsec:bar_dt_grad_bound}.
\begin{lemma}
\label{lem:bar_dt_grad-bound}
Consider Algorithm \ref{alg:(N)-HARPG}. For all $t \geq 1$, we have
   \begin{align}
   \label{eq:bar_dt_grad-bound-lem}
 \E \| \bar{d}_{t} - \nabla J_H(\theta_{t})  \|^2 &\leq  \frac{C_2(1+ \log t)}{(N-f)t},
\end{align}
where $C_2$ is defined in Appendix \ref{subsec:bar_dt_grad_bound}. %$C_2 \coloneqq 2\sigma^2  + 12 (2 G_2^2 + \sigma_h^2 + D_h^2 \g^{2 H}) \cdot \frac{6G_1}{\mu_F} +  24 D_g^2$. 
\end{lemma}

From \eqref{eq:final-dt-dt-bar-bound} and \eqref{eq:bar_dt_grad-bound-lem}, we have
\begin{align}
\label{eq:dt_grad_bound}
    &\E\|d_t - \nabla J_H(\theta_t)\| \leq \frac{4e^2 C_1 \lambda}{t+1} + 2C_1\lambda \sqrt{\frac{2\log (N-f)(t+1)}{t+1}} + \sqrt{\frac{C_2(1+ \log t)}{(N-f)t}}. 
\end{align}
{\bf 3. Obtaining the final bound:} 

Using Lemma 7 of \cite{fatkhullin}, we have
\begin{align}
\label{eq:fatkhullin-lemma-general}
\begin{split}
   &J^* - J(\theta_{t+1}) \\
   &\leq \left(1 - \frac{\sqrt{2 \mu} \g_t}{3}\right)  (J^* - J(\theta_{t}))  + \frac{\varepsilon'\g_t}{3} + \frac{8 \g_t }{3} \mathbb E\|d_t - \nabla J_H(\theta_t)\| + \frac{L \g_t^2}{2} + \frac{4}{3} \g_t D_g \g^{{H}},
\end{split}
\end{align}
where $\mu = \frac{\mu_F^2}{2G_1^2}$, $\varepsilon' = \frac{\sqrt{\varepsilon_{\mathrm{bias}}}}{\mu(1-\gamma)}$ and $D_g, L$ are defined in Appendix \ref{subsec:notations}.

Substituting  \eqref{eq:dt_grad_bound} in \eqref{eq:fatkhullin-lemma-general} and unrolling the recursions gives us the following bound. The full proof can be found in Appendix \ref{subsec:final-bound}.

\begin{lemma}
\label{lem:final-bound}
Consider Algorithm \ref{alg:(N)-HARPG}. Then for all $T \geq 1$
    \begin{align*}
    J^* - J(\theta_{T})  \leq& 
    \frac{\sqrt{\varepsilon_{\mathrm{bias}}}}{(1-\gamma)}
    + \frac{J^* - J(\theta_{0})}{(T+1)^2} + \frac{C_3}{T+1} \\
    &+ \frac{32 G_1}{\mu_F \sqrt{T+1}} \Bigg(\sqrt{\frac{C_2(1+ \log T)}{N-f}} + C_1 \lambda \sqrt{2\log (N-f)(T+1)} \Bigg),
\end{align*}
where $C_3$ is defined in Appendix \ref{subsec:final-bound}.
\end{lemma}

% We can express the final bound in terms of $\gamma$, $\mu_F$, the bound on the reward function $R$ and the Lipschitz and smoothness constants of the score function $G_1$, $G_2$. The above quantities can be used to bound the remaining terms such as the Lipschitz constant $L$, variance bounds $\sigma^2$ and $\sigma_H^2$ and the remaining terms $G_g$, $G_H$, $D_h$ and $D_g$. Substituting these bounds, we obtain the number of samples per agent to achieve $J^* - J(\theta_{T})  \leq \frac{\sqrt{\varepsilon_{\mathrm{bias}}}}{1-\gamma} + \epsilon$ as (details provided in Appendix \ref{subsec:final-bound}):
% \begin{align}
% \label{eq:final-bound-sc-all-terms}
% \begin{split}
% T &=  \mathcal{O} \bigg( \frac{1}{\epsilon^2} \log\left(\frac{1}{\epsilon}\right) \bigg( \frac{1}{(N-f)} \cdot\left(\frac{R^2G_2^2G_1^7}{\mu_F^3(1-\gamma)^4}+\frac{R^4(G_1^7+G_1^3G_2^2)}{\mu_F^3(1-\gamma)^8}\right) \\
% &+ \lambda^2 \log(N-f) \cdot \bigg(\frac{R^2G_2^2G_1^6}{\mu_F^2(1-\gamma)^4}+\frac{R^4(G_1^6+G_1^2G_2^2)}{\mu_F^2(1-\gamma)^8}\bigg)\bigg)\bigg).
% \end{split}
% \end{align}

\begin{figure*}[htbp]
\centering
\subfigure[CartPole (Random Noise)]{
\label{fig:1(a)} 
\includegraphics[width=2.1in, height=1.3in]{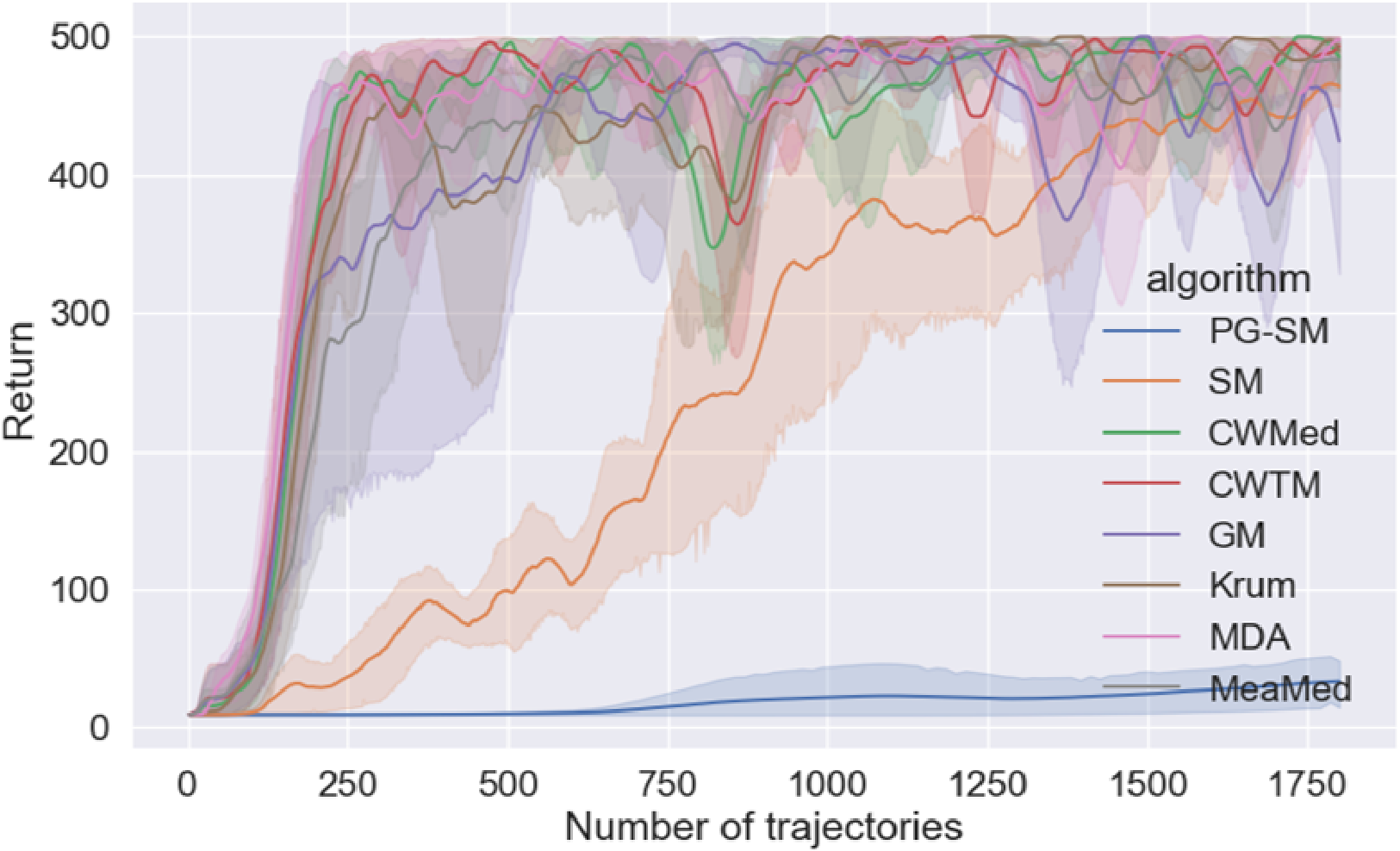}}
\subfigure[CartPole (Random Action)]{
\label{fig:1(b)} 
\includegraphics[width=2.1in, height=1.3in]{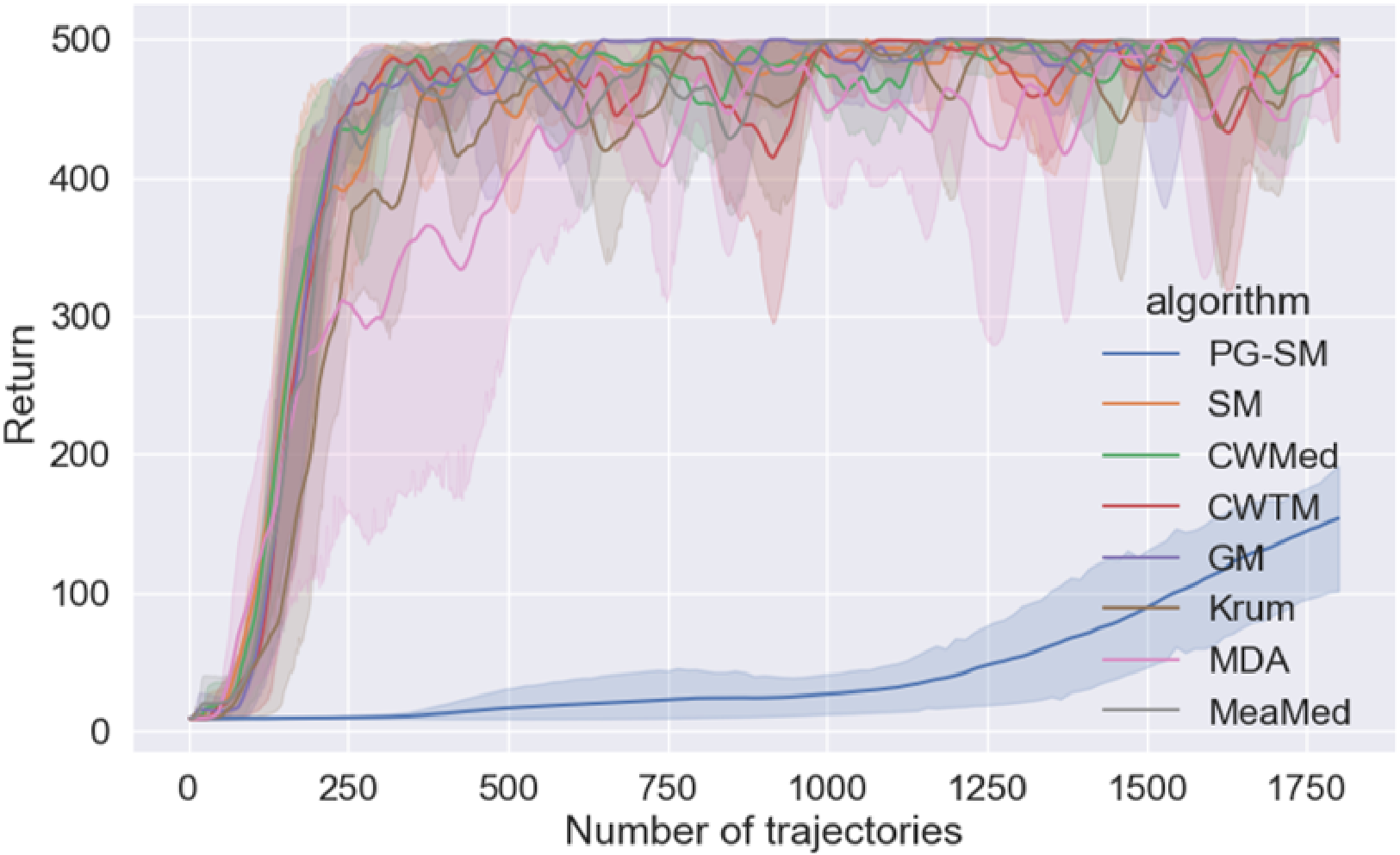}}
\subfigure[CartPole (Sign Flipping)]{
\label{fig:1(c)} 
\includegraphics[width=2.1in, height=1.3in]{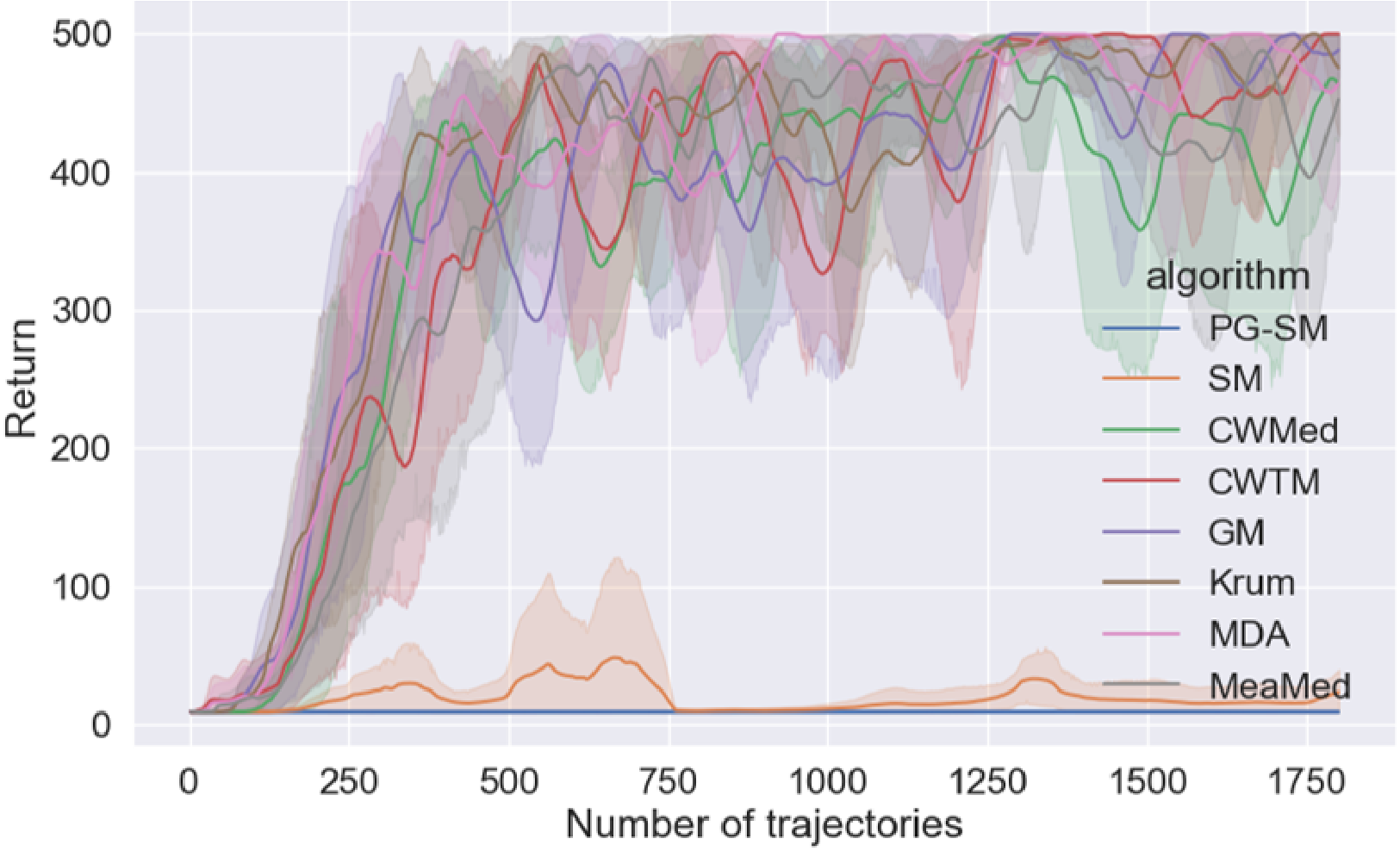}}
\subfigure[InvertedPendulum (Random Noise)]{
\label{fig:1(d)} 
\includegraphics[width=2.1in, height=1.3in]{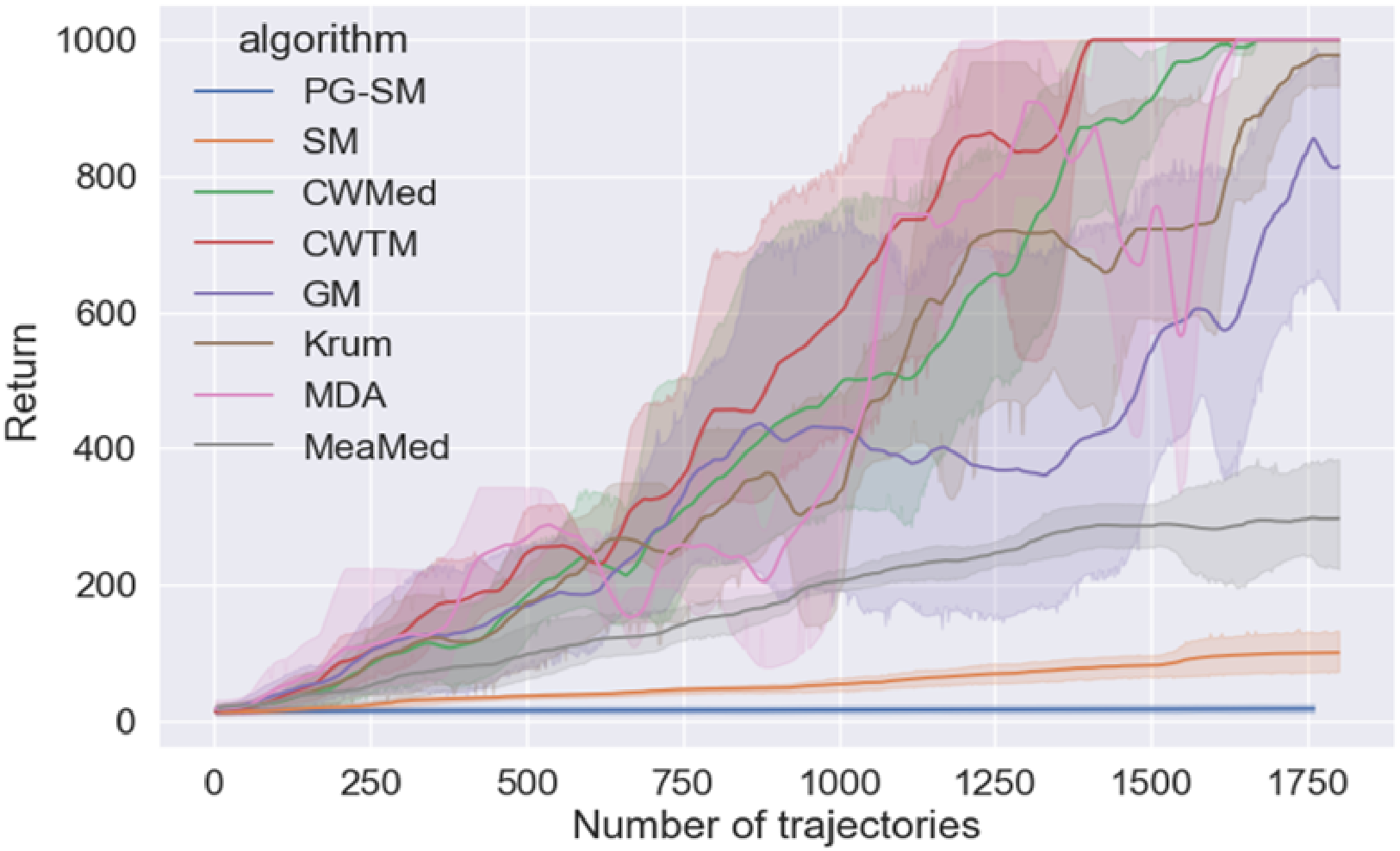}}
\subfigure[InvertedPendulum (Random Action)]{
\label{fig:1(e)} 
\includegraphics[width=2.1in, height=1.3in]{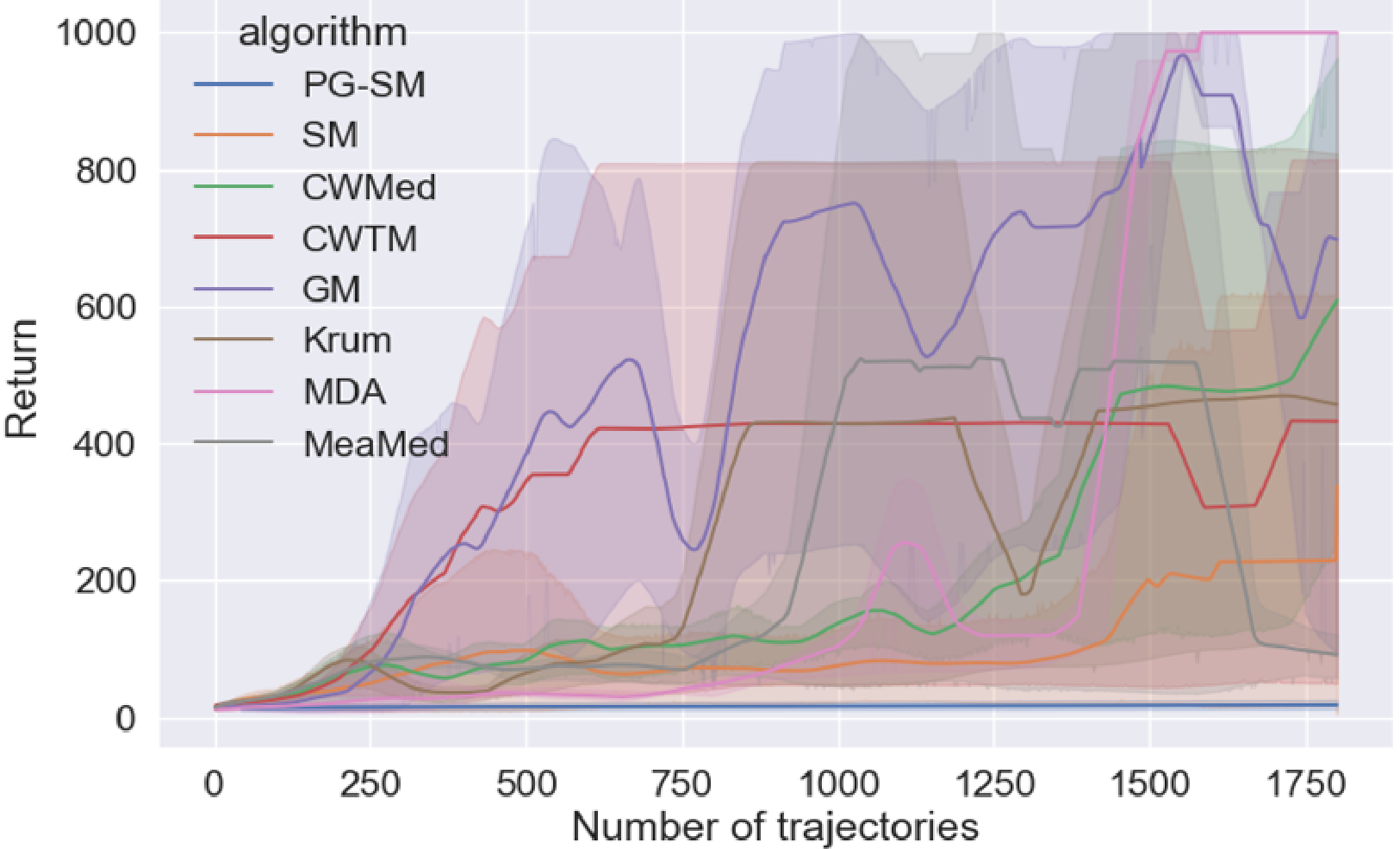}}
\subfigure[InvertedPendulum (Sign Flipping)]{
\label{fig:1(f)} 
\includegraphics[width=2.1in, height=1.3in]{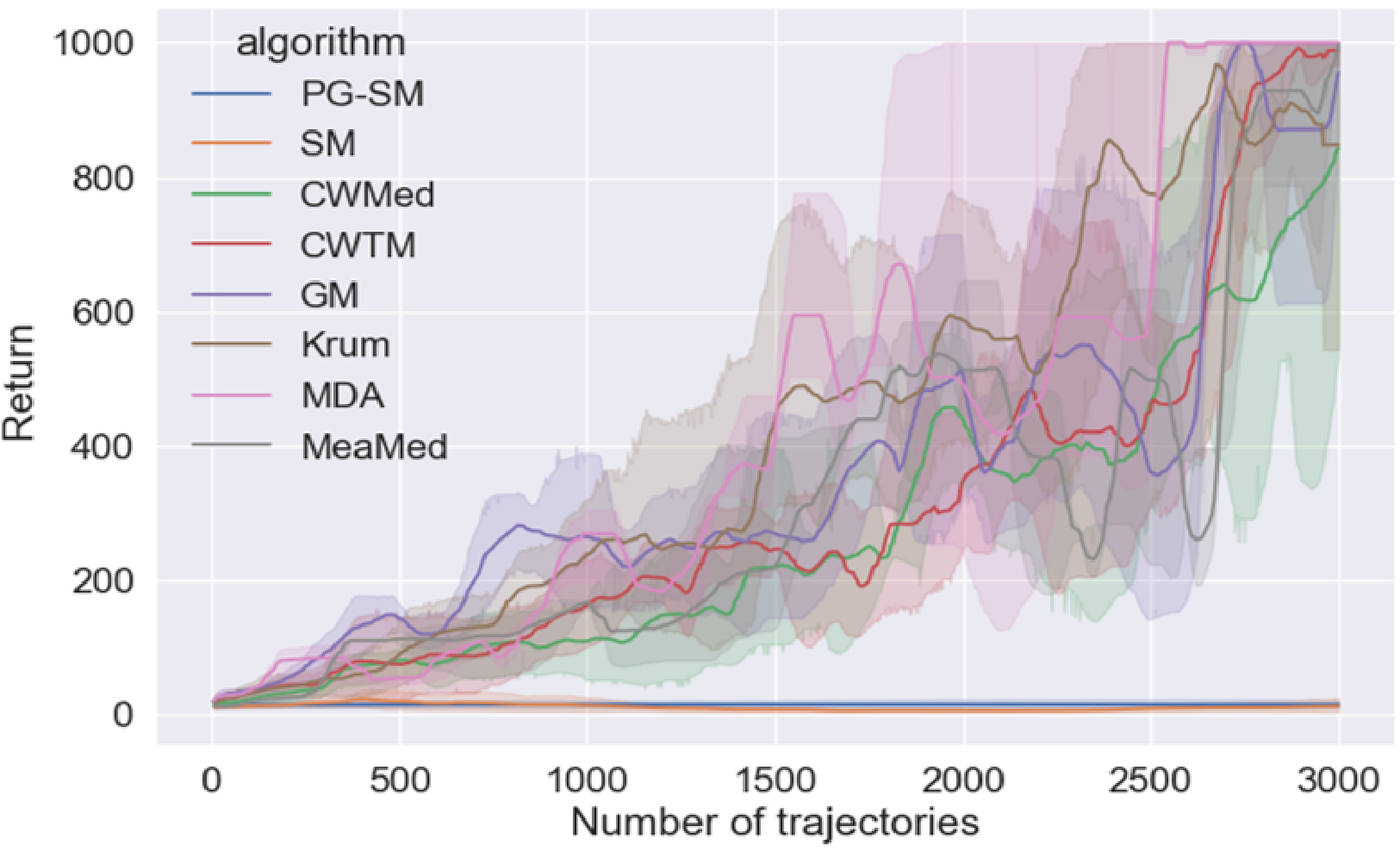}}
\vspace{-.1in}
\caption{\small Evaluation results of Res-NHARPG on CartPole and InvertedPendulum. We test Res-NHARPG with the six aggregators as shown in Table \ref{table:sample complexities}. For baselines, we select Res-NHARPG with a simple mean (SM) function as the aggregator, which is equivalent to the original N-HARPG algorithm, and a vanilla policy gradient method with the simple mean aggregator (PG-SM). For each environment, there are ten workers, of which three are adversaries, and we simulate three types of attacks: random noise, random action, and sign flipping. It can be observed that N-HARPG outperforms PG and Res-NHARPG with those $(f, \lambda)$ aggregators can effectively handle multiple types of attacks during the learning process.% Note that each algorithm in each subfigure is run five times with different random seeds, for which the mean and 95\% confidence interval are shown as the solid line and shadow area, respectively.
}
\label{fig:1} 
\vspace{-.2in}
\end{figure*}

\section{Evaluation}

To show the effectiveness of our algorithm design (i.e., Res-NHARPG), we provide evaluation results on two commonly-used continuous control tasks: CartPole-v1 from OpenAI Gym \cite{DBLP:journals/corr/BrockmanCPSSTZ16} and InvertedPendulum-v2 from MuJoCo \cite{DBLP:conf/iros/TodorovET12}. Additional experiments on more demanding MuJoCo tasks, including HalfCheetah, Hopper, Inverted Double Pendulum, and Walker are provided in Appendix \ref{sec:additional-exp}. For each task on CartPole-v1 and InvertedPendulum-v2, there are ten workers to individually sample trajectories and compute gradients, and three of them are adversaries who would apply attacks to the learning process. Note that we do not know which worker is an adversary, so we cannot simply ignore certain gradient estimates to avoid the attacks. We simulate three types of attacks to the learning process: random noise, random action, and sign flipping. 

%As shown in Figure \ref{fig:1}, we compare among Res-NHARPG with the six aggregators shown in Table \ref{table:sample complexities} (i.e., MDA, CWTM, CWMed, Krum, MeaMed, GM), Res-NHARPG with a simple mean (i.e., SM) aggregator, and Vanilla Policy Gradient with SM (i.e., PG-SM). Specifically, the simple mean aggregator only averages the estimates of gradients from all workers without considering the existence of adversaries, so Res-NHARPG with SM is equivalent to the state-of-the-art policy gradient method -- N-HARPG \cite{fatkhullin}. Through these comparisons, we expect to show the necessity to utilize these $(f, \lambda)$ aggregators in case of adversaries and the robustness of them for various types of environments and attacks.

In Figure \ref{fig:1}, we present the learning process of the eight algorithms in two environments under three types of attacks. In each subfigure, the x-axis represents the number of sampled trajectories; the y-axis records the acquired trajectory return of the learned policy during evaluation. Each algorithm is repeated five times with different random seeds. The average performance and 95\% confidence interval are shown as the solid line and shadow area, respectively. Codes for our experiments have been submitted as supplementary material and will be made public. 

Comparing the performance of N-HARPG (i.e., SM) and Vanilla PG (i.e., PG-SM), we can see that N-HARPG consistently outperforms, especially in Figure \ref{fig:1(a)} and \ref{fig:1(b)} of which the task and attacks are relatively easier to deal with. For the `random action' attack (Figure \ref{fig:1(b)} and \ref{fig:1(e)}), which does not directly alter the gradient estimates, N-HARPG shows better resilience. However, in more challenging tasks (e.g., InvertedPendulum) and under stronger attacks (e.g., sign flipping), both N-HARPG and Vanilla PG would likely fail, which calls for effective aggregator functions.

For CartPole, the maximum trajectory return is set as 500. Res-NHARPG, implemented with each of the six aggregators, can reach that expert level within 1000 trajectory samples, with slight difference in the convergence rate. As for InvertedPendulum, not all aggregators achieve the expert level (i.e., a trajectory return of 1000), yet they all demonstrate superior performance compared to those employing only the simple mean aggregator. It's worth noting that Res-NHARPG with the MDA aggregator consistently converges to the expert level across all test cases, showing its robustness. Moreover, the `random action' attack brings more challenges to the aggregators, as the influence of random actions during sampling on the gradient estimates is indirect while all aggregators filter abnormal estimates based on gradient values. 
%{\bf \color{red} Mention about supplementary experiments on more demanding MuJoCo tasks, including HalfCheetah, Hopper, Inverted Double Pendulum, and Walker and add in Appendix. }

\section{Summary}

In this paper, we investigate the impact of adversaries on the global convergence sample complexity of Federated Reinforcement Learning (FRL). We introduce Res-NHARPG, and show its sample complexity is of order $\tilde{\mathcal{O}}\left( \frac{1}{\epsilon^2} \left( \frac{1}{N-f} + \lambda^2 \right)\right)$ using $(f,\lambda)$-aggregators, where $N$ is the total number of workers and $f$ is the number of faulty workers. Notably, when certain aggregators are used (MDA, CWTM and MeaMed), we show our approach achieves optimal sample complexity. 

This work opens up multiple possible future directions in RL with adversaries, including delay in agent feedback and heterogeneous agents. {\color{black}We further note that a parameter free version of the proposed algorithm is also an important future direction. }

%There are several problems in RL with Byzantine adversaries that have not been explored e.g., impact of adversaries on asynchronous RL algorithms and heterogenous RL environments. We believe that our current analysis could be useful for studying these problems as well.

\section*{Acknowledgment}
Swetha Ganesh's research is supported by the Overseas Visiting Doctoral Fellowship (OVDF) and Prime Minister's Research Fellowship (PMRF). Gugan Thoppe's research is supported by the Indo-French Centre for the Promotion of Advanced Research---CEFIPRA (7102-1), the Walmart Centre for Tech Excellence at IISc (CSR Grant WMGT-23-0001),  DST-SERB's Core Research Grant (CRG/2021/008330), and the Pratiksha Trust Young Investigator Award.

\bibliographystyle{tmlr}

 \bibliography{references}

\newpage

\renewcommand\thesection{\Alph{section}}
\setcounter{section}{0} 

\appendix
\onecolumn

%It's worth noting that \cite{fatkhullin} has made further advancements in sample complexity results, which will serve as a foundational reference for this paper in the absence of adversarial elements. We note that improved sample complexity using second order algorithms have been studied in as well, while extensions of such approaches to federated setup is not straightforward, and is thus a subject of future work.
\section{Additional experiments}
\label{sec:additional-exp}

\begin{figure*}[htbp]
\centering
\subfigure[HalfCheetah]{
\label{fig:2(a)} 
\includegraphics[width=2.6in, height=1.4in]{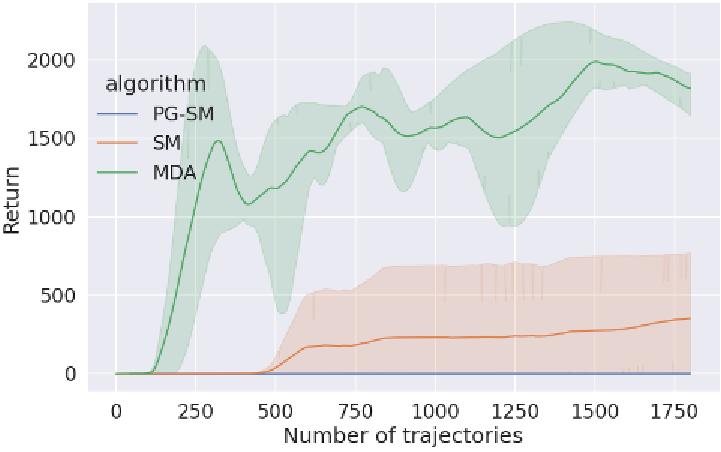}}
\subfigure[Hopper]{
\label{fig:2(b)} 
\includegraphics[width=2.6in, height=1.4in]{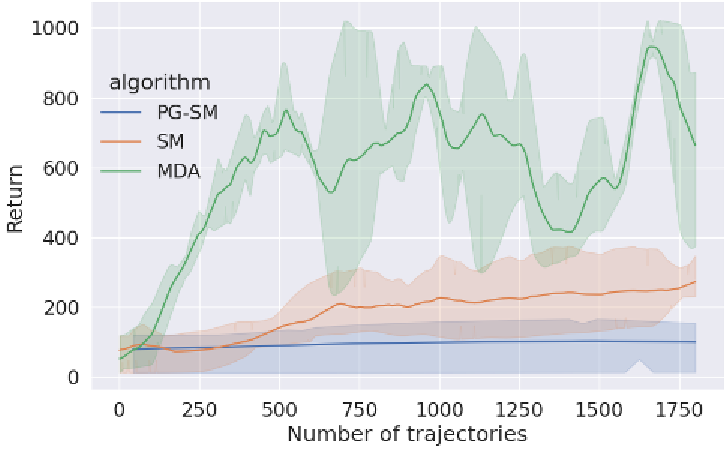}}
\subfigure[Inverted Double Pendulum]{
\label{fig:2(c)} 
\includegraphics[width=2.6in, height=1.4in]{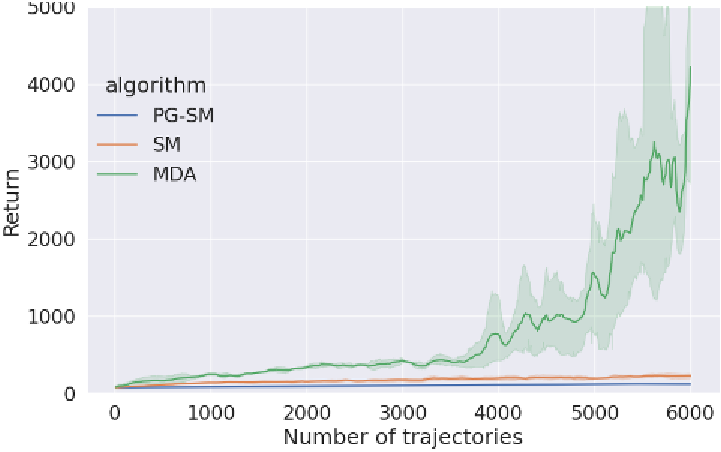}}
\subfigure[Walker]{
\label{fig:2(d)} 
\includegraphics[width=2.6in, height=1.4in]{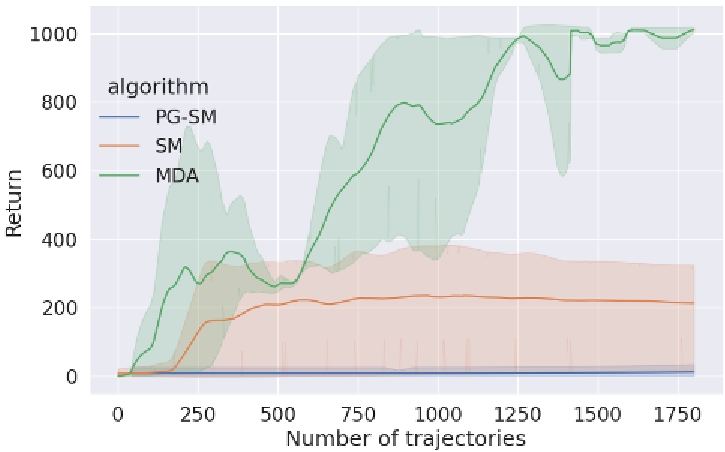}}
\caption{Res-NHARPG with the MDA aggregator consistently outperform the baselines: N-HARPG (i.e., SM) and Vanilla PG (i.e., PG-SM), on a series of MuJoCo tasks.}
\label{fig:2} 
\end{figure*}

Previous evaluation results have shown the superiority of Res-NHARPG with $(f,\lambda)$- aggregators. To further demonstrate its applicability, we consider Res-NHARPG with MDA and compare it with the baselines: N-HARPG (i.e., SM) and Vanilla PG (i.e., PG-SM), on a series of more challenging MuJoCo tasks: HalfCheetah, Hopper, Inverted Double Pendulum, Walker, of which the result is shown as Figure \ref{fig:2}. Our algorithm consistently outperforms the baselines when adversaries (specifically, random noise) exist, and relatively, N-HARPG performs better than Vanilla PG. Note that our purpose is not to reach SOTA performance but to testify the effectiveness of aggregators, so the three algorithms in each subfigure share the same set of hyperparameters (without heavy fine-tuning). In Figure \ref{fig:2}, we illustrate the training progress, up to a maximum number of sampled trajectories (6000 for the Inverted Double Pendulum and 1800 for other tasks), for each algorithm by plotting their episodic returns. The advantage of Res-NHARPG is more significant when considering the peak model performance. For instance, the highest evaluation score achieved by Res-NHARPG on Inverted Double Pendulum can exceed 9000, i.e., the SOTA performance as noted in \cite{tianshou}, while the baselines' scores are under 500. {\color{black} We also evaluate the Fed-ADMM algorithm proposed in \cite{lan2023improved} in Fig. \ref{fig:3}. As expected, even with a higher number of samples, Fed-ADMM yields significantly lower returns compared to Res-NHARPG with the MDA aggregator (see, for instance, Fed-ADMM performs poorly in HalfCheetah, achieves a return of 350 in Hopper compared to consistently above 400 for our algorithm; 90 in Inverted Double Pendulum to above 4000 for our algorithm; less than 300 in Walker to around 1000 for our algorithm). }

\textbf{Details regarding attacks considered:} By `random noise' or `sign flipping', the real estimated policy gradients are altered by adding random noises or multiplying by a negative factor, respectively. While, for `random action', adversarial workers would select random actions at each step, regardless of the state, when sampling trajectories for gradient estimations. Unlike the other attacks, `random action' does not directly change the gradients, making it more challenging to detect. Also, `random action' is different from the widely-adopted $\epsilon$-greedy exploration method, since the action choice is fully random and the randomness does not decay with the learning process.

\begin{figure*}[t]
\centering
\subfigure[HalfCheetah]{
\label{fig:3(a)} 
\includegraphics[width=2.4in, height=1.4in]{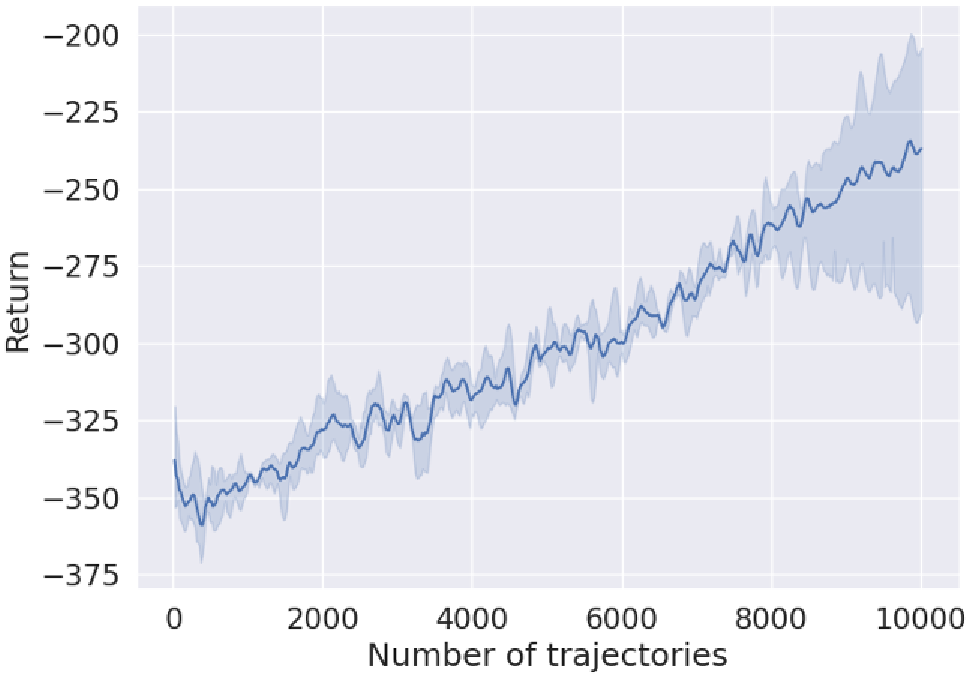}}
\subfigure[Hopper]{
\label{fig:3(b)} 
\includegraphics[width=2.4in, height=1.4in]{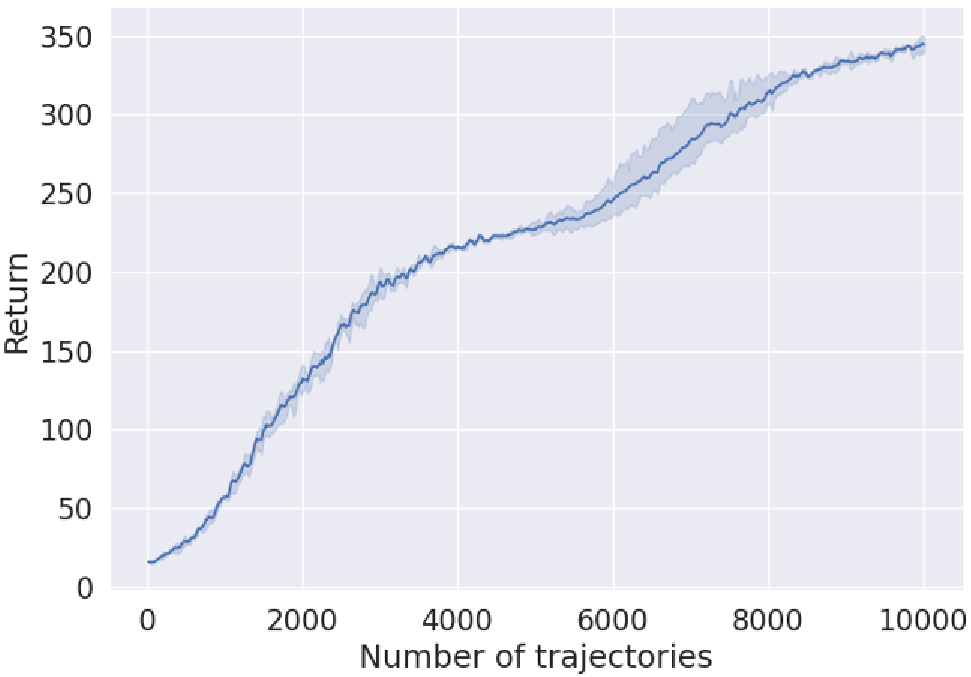}}
\subfigure[Inverted Double Pendulum]{
\label{fig:3(c)} 
\includegraphics[width=2.4in, height=1.4in]{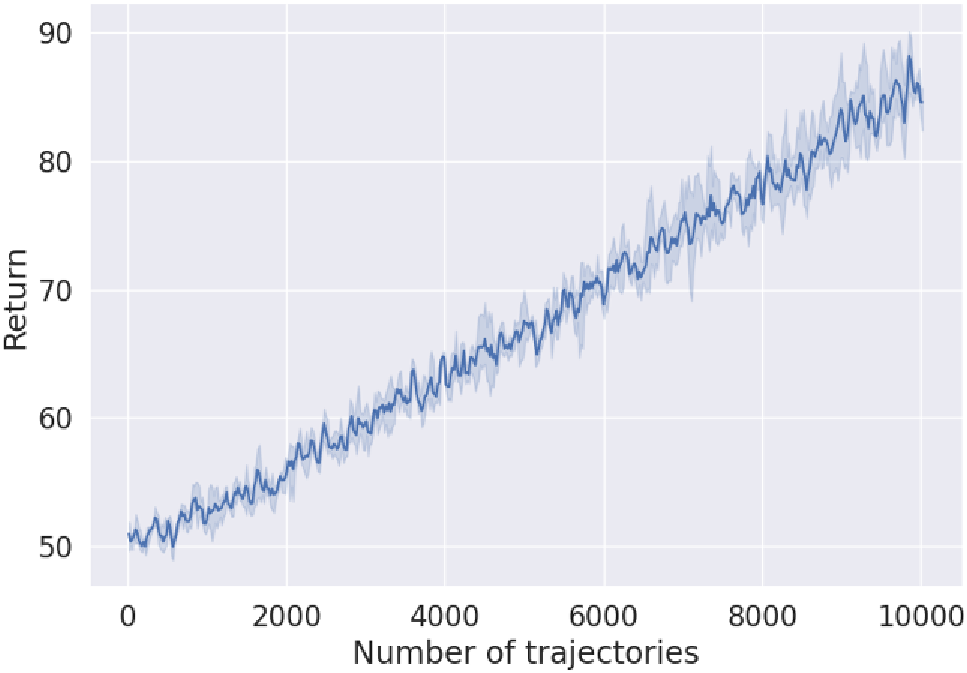}}
\subfigure[Walker]{
\label{fig:3(d)} 
\includegraphics[width=2.4in, height=1.4in]{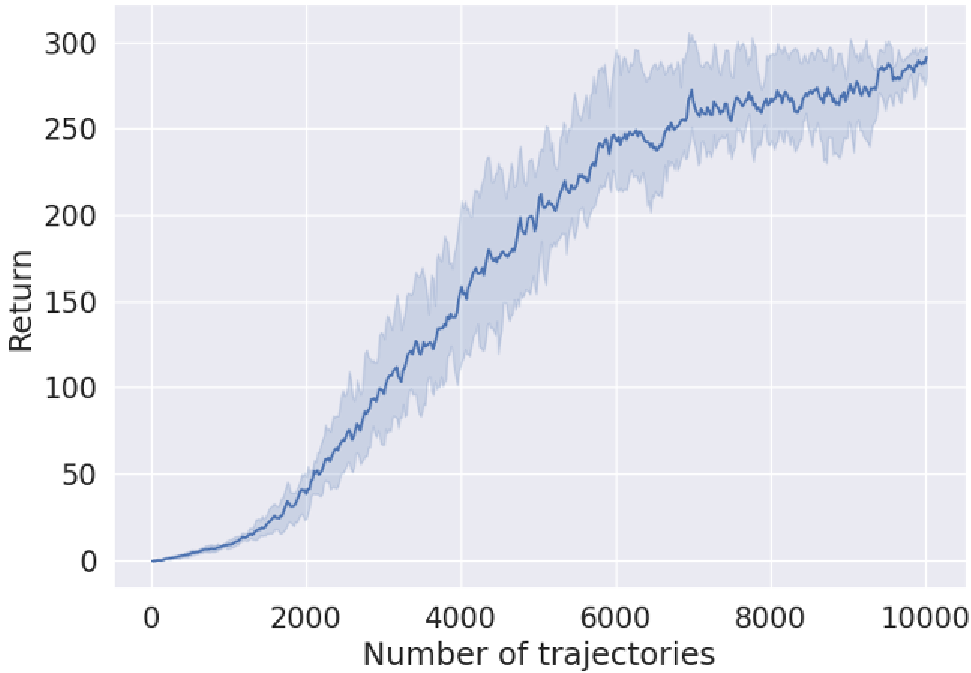}}
\caption{\color{black}Evaluation of Fed-ADMM \cite{lan2023improved} on MuJoCo tasks with random noise. The solid lines represent the mean performance, while the shaded areas indicate the 95\% confidence intervals from repeated experiments. We used the official implementation from \cite{lan2023improved}.}
\label{fig:3} 
\end{figure*}

\section{Details of aggregator functions}
\label{sec:details-aggregators}
Let $[x]_k$ denote the $k$-th coordinate of $x \in \Rd$. Given a set of $n$ vectors $X=\{x_1, \ldots, \, x_n\}$ as input, the outputs of different aggregator functions are given below.
\begin{itemize}
\item {\bf Minimum Diameter Averaging (MDA)} (originally proposed in~\cite{rousseeuw1985multivariate} and reused in~\cite{mhamdi18a}) selects a set $\cG^{*}$ of cardinality $N-f$ with the smallest diameter, i.e., 
\begin{equation*}
    \cG^{ *} \in \argmin_{\underset{|\cG| = N-f}{\cG \subset \{ 1, \ldots, \, N} \} } \left\{\max_{i,j\in \cG} \|x_i - x_j\| \right\} \label{eqn:def_S*}
\end{equation*}
and outputs $\frac{1}{N-f} \sum_{i \in \cG^{*}} x_i$.
\item {\bf Co-ordinate wise Trimmed Mean (CWTM)} \cite{yin2018byzantine}:
%Let $\tau_k$ denote a permutation on $[n]$ that sorts the $k$-coordinate of the input vectors in non-decreasing order, i.e., $[x_{\tau_k(1)}]_k\leq [x_{\tau_k(2)}]_k \leq\ldots \leq [x_{\tau_k(n)}]_k$. Then, the CWTM of $x_1, \ldots, \, x_n$, denoted by $\text{CWTM}(x_1, \ldots, x_n)$, is a vector in $\R^d$ whose $k$-th coordinate is defined as follows,
% \begin{equation*}
%     \left[\text{CWTM}(x_1, \ldots, x_n)\right]_k := \frac{1}{n-2f} \sum_{j \in [f+1,N-f]} [x_{\tau_k(j)}]_k. \label{eqn:def_cwtm}
% \end{equation*}
Consider co-ordinate $k=1,2,\cdots,d$. Let $\cG_k \subset X$ be such that $\cG^{c}_k$ consists only of elements in $X$ with the $f$ largest or $f$ smallest values of $[x]_k$. Then, CWTM outputs
\begin{equation*}
    \left[\text{CWTM}(x_1, \ldots, x_n)\right]_k = \frac{1}{N-2f} \sum_{x \in \cG_k} [x]_k. \label{eqn:def_cwtm}
\end{equation*}

\item {\bf Co-ordinate wise Median (CWMed)} \cite{yin2018byzantine}:
The output of CWMed is given by
\begin{equation*}
    \left[\text{CWMed}\left(x_1, \ldots, x_n\right)\right]_k =\text{Median}\left( [x_1]_k, \ldots [x_n]_k \right). \label{eqn:def_cwmed}
\end{equation*}

\item {\bf Mean around Median (MeaMed)} \cite{meamed}:
\label{sec:Meamed}
MeaMed computes the average of  the $N-f$ closest elements to the median in each dimension. Specifically, for each $k \in [d]$, $m \in [n]$, let $i_{m;k}$ be the index of the input vector with $k$-th coordinate that is $m$-th closest to $\textnormal{Median}([x_1]_k,\dots,[x_n]_k)$. Let $C_k$ be the set of $N-f$ indices defined as
\[C_k = \{i_{1;k}, \ldots, \, i_{N-f;k}\}.\] Then we have
\begin{equation*}
    [\text{MeaMed}(x_1, \ldots, x_n)]_k = \frac{1}{N-f} \sum_{i \in C_{k}} [x_i]_k.
\end{equation*}

    \item {\bf Krum} \cite{krum}: Multi-Krum$^*$ outputs an average of the vectors that are the closest to their neighbors upon discarding $f$ farthest vectors. Specifically, for each $i \in [n]$ and $k \in [n-1]$, let $i_k \in [n] \setminus \{i\}$ be the index of the $k$-th closest input vector from $x_i$, i.e., we have $\|x_i - x_{i_1}\| \leq \ldots \leq \|x_i - x_{i_{n-1}}\|$ with ties broken arbitrarily. Let $C_i$ be the set of $N-f-1$ closest vectors to $x_i$, i.e., 
\[C_i = \{i_{1}, \ldots, \, i_{N-f-1}\}.\] Then, for each $i \in [n]$, we define $score(i) = \sum_{j\in C_i} \|x_i - x_j\|^2 $. 
Finally, Multi-Krum$^*_q$ outputs the average of $q$ input vectors with the smallest scores, i.e., 
\begin{align*}
    \text{Multi-Krum}^{*}_q \left(x_1, \ldots, \, x_n \right) = \frac{1}{q} \sum_{i \in  S(q)} x_i,
    \label{eqn:def_krum}
\end{align*}
where $ S(q)$ is the set of $q$ vectors with the smallest scores. We call Krum$^*$ the special case of Multi-Krum$^*_q$ for $q = 1$.

    \item {\bf Geometric Median (GM)} \cite{chen2017distributed}: For input vectors $x_1, \ldots, \, x_n$, their geometric median, denoted by $\text{GM}(x_1, \ldots, x_n)$, is defined to be a vector that minimizes the sum of the distances to these vectors. Specifically, we have
 \begin{equation*}
     \textnormal{GM}(x_1, \ldots, x_n) \in \argmin_{z \in \R^d} \sum_{i = 1}^n \|z-x_i\|.
 \end{equation*}
 
\end{itemize}

\textbf{Comments:}

\begin{itemize}
\item We observe that MDA, CWTM and MeaMed achieve optimal bounds, while CWMed, Krum and GM do not. This is explained in \cite{farhadkhani22a} by noting that Krum, CWMed, and GM operate solely on median-based aggregation, while MDA, CWTM, and MeaMed perform an averaging step after filtering out suspicious estimates. This averaging step results in variance reduction, similar to simple averaging in vanilla distributed SGD. % While median-based aggregators suffer from increase in variance with the sample size, as indicated by conventional bounds derived from order statistics \cite{Bertsimas_Natarajan_Teo_2006}.
\item It's worth mentioning that CWMed and GM do not require any knowledge of $f$ to be implemented. On the other hand, while MDA, CWTM, MeaMed, and Krum do rely on some knowledge of $f$, they can still be implemented by substituting an upper bound for $f$ instead. In this case, our guarantees would scale based on this upper bound, instead of $f$.
\end{itemize}

\section{Proof details}
\subsection{Notations}
\label{subsec:notations}
Before proceeding further, we begin with the following lemma where we also introduce a few notations.
\begin{lemma}
\label{lem:general-lem}
For all $\theta \in \R^d$ and trajectories $\tau$ of length $H$ sampled by policy $\pi_{\theta}$, we have 
\begin{enumerate}
    \item $J$ is $L$-smooth with $L \coloneqq \frac{R(G_1^2+G_2)}{(1-\gamma)^2}$.
    
    \item $g(\tau,\theta)$ is an unbiased estimate of $\nabla J_H(\theta)$ and \begin{align}
         \|g(\tau,\theta)\| \leq \frac{G_1 R}{(1-\gamma)^2}
         \coloneqq G_g.
    \end{align}
    \item $B(\tau,\theta)$ is an unbiased estimate of $\nabla^2 J_H(\theta)$ and  \begin{align}
    \label{eq:hessian-est-bound}
         \|B(\tau,\theta)\| \leq \frac{G_2G_1^2R+LR}{(1-\gamma)^2} \coloneqq G_H.
    \end{align}
    \item  For some $0 < \sigma^2_H \leq G_H^2$, we have
    \begin{align}
         \E\|B_t(\tau,\theta) - \nabla^2 J_H(\theta)\|^2 \leq \sigma_H^2.
    \end{align}
   
\end{enumerate}
        
\end{lemma}

The first two statements are given in Lemma 4.2 and 4.4, \cite{yuan22a} and the third and fourth statements are given in Lemma 4.1 in \cite{shen19d}. Since the algorithm makes use of the truncated gradient and hessian estimates $\nabla J_H(\theta)$ and $\nabla^2 J_H(\theta)$ instead of $\nabla J(\theta)$ and $\nabla^2 J(\theta)$, we must bound the error terms $\nabla J(\theta)-\nabla J_H(\theta)$ and $\nabla^2 J(\theta)-\nabla^2 J_H(\theta)$. It is known that these error terms vanish geometrically fast with $H$ and is stated below:
\begin{lemma}[Lemma 3, \cite{masiha}]
\label{masiha-lemma}
    Let Assumption \ref{assump: conditions on score function} be satisfied, then for all $\theta \in \R^d$ and every $H\geq 1$, we have
    \begin{align*}
        \|\nabla J_H(\theta)-\nabla J(\theta)\| \leq D_g \gamma^H \text{ , } \|\nabla J_H^2(\theta)-\nabla J^2(\theta)\| \leq D_h \gamma^H,
    \end{align*}
    where $D_g \coloneqq \frac{G_1 R}{1-\gamma} \sqrt{\frac{1}{1-\gamma}+H}$ and $D_h \coloneqq \frac{R(G_2 +G_1^2)}{1-\gamma} \left(\frac{1}{1-\gamma}+H\right)$.
\end{lemma}

\subsection{Proof of Lemma \ref{lem:bar_dt_grad-bound}}
\label{subsec:bar_dt_grad_bound}

The proof follows similarly as in \cite{fatkhullin}. Here, we take into account the affect of estimates of $(N-f)$ agents versus one agent. Since each of the $d_t^{(n)}$ are updated as
\begin{align}
  d_t^{(n)} &= (1 - \eta_t) ( d_{t-1}^{(n)} +  B(\hat{\tau}_t^{(n)}, \hat{\theta}_{t}^{(n)})(\theta_t - \theta_{t-1}) ) + \eta_t g(\tau_t^{(n)}, \theta_t^{(n)}),
\end{align}
 we notice that $\bar{d}_t = \frac{1}{N-f} \sum_{n \in \cG} d_t^{(n)} $ can be expressed as
\begin{align}
  \bar{d}_t &= (1 - \eta_t) ( \bar{d}_{t-1} +  \bar{B}(\hat{\theta}_{t})(\theta_t - \theta_{t-1}) ) + \eta_t \bar{g}(\theta_t),
\end{align}
where $\bar{B}(\hat{\theta}_{t}) \coloneqq \frac{1}{N-f} \sum_{n \in \cG} B(\hat{\tau}_t^{(n)}, \hat{\theta}_{t}^{(n)})$ and $\bar{g}(\theta_t) \coloneqq \frac{1}{N-f} \sum_{n \in \cG} g(\tau_t^{(n)}, \theta_{t})$.

It follows that
\begin{align}
\label{eq:d_t-expansion}
  \bar{d}_t - \nabla J_H(\theta_t) &= (1 - \eta_t) ( \bar{d}_{t-1} +  \bar{B}(\hat{\theta}_{t})(\theta_t - \theta_{t-1}) ) + \eta_t \bar{g}(\theta_t) - \nabla J_H(\theta_t) \nonumber \\
  &= (1 - \eta_t) ( \bar{d}_{t-1} - \nabla J_H(\theta_{t-1}) +  \bar{B}(\hat{\theta}_{t})(\theta_t - \theta_{t-1}) + \nabla J_H(\theta_{t-1}) - \nabla J_H(\theta_t))) \nonumber \\
  &+ \eta_t (\bar{g}(\theta_t) - \nabla J_H(\theta_t)) \nonumber \\
  &= (1 - \eta_t) ( \bar{d}_{t-1} - \nabla J_H(\theta_{t-1})) +  (1 - \eta_t) \bar{\mathcal{W}}_t  + \eta_t \bar{\mathcal{V}}_t,
\end{align}
where $\bar{\mathcal{W}}_t \coloneqq \bar{B}(\hat{\theta}_{t})(\theta_t - \theta_{t-1}) + \nabla J_H(\theta_{t-1}) - \nabla J_H(\theta_t)$ and $\bar{\mathcal{V}}_t \coloneqq \bar{g}(\theta_t) - \nabla J_H(\theta_t)$. We then have the following lemma

\begin{lemma}
\label{lem:variance-N-f}
For all $t \geq 1$, the following statements hold
\begin{enumerate}[label=(\roman*)]
    \item $ \E[\bar{\mathcal{W}}_t] = 0$
    \item $\E[\bar{\mathcal{V}}_t] = 0$
    \item $\E[\|\bar{\mathcal{V}}_t\|^2] \leq \frac{\sigma^2}{N-f}$
    \item $\E[\|\bar{\mathcal{W}}_t\|^2] \leq \frac{12}{N-f} ((2 L^2 + \sigma_H^2 + D_h^2 \g^{2 H})\cdot \gamma_t^2 + D_g^2 \g^{2 H})$.
\end{enumerate}
\end{lemma}
\begin{proof}[Proof of Lemma \ref{lem:variance-N-f}]
    The first two statements are easy to see since $\E[\bar{\mathcal{V}}_t] = \frac{1}{N-f}\sum_{n \in \cG} \E[\mathcal{V}_t^{(n)}]$ and $\E[\mathcal{V}_t^{(n)}]=0$ for all $n \in \cG$ from Lemma \ref{lem:general-lem}. Similar arguments show that $\E[\bar{\mathcal{W}}_t] = 0$.

    Notice that
    \begin{align}
        \E[\|\bar{\mathcal{V}}_t\|^2] = \frac{1}{(N-f)^2}\sum_{n \in \cG} \E[\|\mathcal{V}_t^{(n)}\|^2] \leq  \frac{\sigma^2}{N-f},
    \end{align}
    where the last inequality follows from Assumption \ref{assump: variance}. Similarly,
    \begin{align}
        \E[\|\bar{\mathcal{W}}_t\|^2] = \frac{1}{(N-f)^2}\sum_{n \in \cG} \E[\|\mathcal{W}_t^{(n)}\|^2].
    \end{align}
    We have for all $n \in \cG$
    \begin{align*}
        \E\| \mathcal{W}_t^{(n)} \|^2  & =  \E \|  \nabla J_{H}(\theta_{t-1}) - \nabla J_{H}(\theta_{t}) + B(\hat{\tau}_t^{(n)}, \hat{\theta}_{t}^{(n)})(\theta_t - \theta_{t-1}) \|  \\
      & \leq   6\E \| \nabla J_H(\theta_{t-1}) - \nabla J_H(\theta_{t-1}) \|^2 +    6\E \|\nabla J_H(\theta_{t-1}) - \nabla J_H(\theta_{t}) \|^2 \\
      &+ 6\E \|B(\hat{\tau}_t^{(n)}, \hat{\theta}_{t}^{(n)}) -  \nabla^2 J_H(\hat{\theta}_{t}) (\theta_t - \theta_{t-1}) \|^2 + 6\E \| \nabla^2 J_{H} (\hat{\theta}_{t}) - \nabla^2 J (\hat{\theta}_{t}) (\theta_t - \theta_{t-1}) \|^2 \\
      &+  6\E \| \nabla^2 J (\hat{\theta}_{t}) (\theta_t - \theta_{t-1})\|^2 \\
      &\leq 2(6L^2+3\sigma^2_H)\E\|\theta_t-\theta_{t-1}\|^2+12D_g^2\gamma^{2H} +6D_h^2\gamma^{2H}\E\|\theta_t-\theta_{t-1}\|^2 \\
      &\leq 2(6L^2+3\sigma^2_H)\cdot \gamma_t^2+12D_g^2\gamma^{2H} +6D_h^2\gamma^{2H}\cdot \gamma_t^2\\
      &=12 ((2 L^2 + \sigma_H^2 + D_h^2 \g^{2 H})\cdot \gamma_t^2 + D_g^2 \g^{2 H}),
    \end{align*}
where the last inequality follows from the fact that $\|\theta_t - \theta_{t-1}\| = \|\gamma_t \frac{d_t}{\|d_t\|}\| = \gamma_t$.
\end{proof}

Using Lemma \ref{lem:variance-N-f}, we obtain 
\begin{align}
\begin{split}
    \E \| \bar{d}_{t} - \nabla J_H(\theta_{t})  \|^2  &\leq  (1 - \eta_t) \E \| \bar{d}_{t-1} - \nabla J_H(\theta_{t-1})  \|^2 \\
    &+ \frac{1}{N-f} (2\sigma^2 \eta_t^2 + 12 ((2 L^2 + \sigma_H^2 + D_h^2 \g^{2 H})\cdot \gamma_t^2 + D_g^2 \g^{2 H}).
\end{split}
\end{align} 

Let $y_t \coloneqq  \frac{1}{N-f} (2\sigma^2 \eta_t^2  + 12 ((2 L^2 + \sigma_H^2 + D_h^2 \g^{2 H})\cdot \gamma_t^2 + D_g^2 \g^{2 H}))$. Unrolling the above, we obtain
\begin{align*}
 \E \| \bar{d}_{t} - \nabla J_H(\theta_{t})  \|^2 &\leq \sum_{i=1}^t \prod_{j=i+1}^t (1 - \eta_j) y_i.
\end{align*}
For $\eta_t = \frac{1}{t}$, it follows that $\prod^{j}_{k=j_0}(1-\eta_k) = \frac{j_0-1}{j}$ and 
\begin{align*}
 \E \| \bar{d}_{t} - \nabla J_H(\theta_{t})  \|^2 &\leq \frac{1}{t} \sum_{i=1}^t i y_i.
\end{align*}

Note that for $\eta_t = \frac{1}{t}$, $\gamma_t = \frac{6G_1}{\mu_F (t+2)}$ and $H = \frac{\log(T+1)}{(1-\gamma)}$, we have $i\eta_i^2 = \frac{1}{i}$, $i\gamma_i^2 \leq \frac{6G_1}{\mu_F i}$ and $i\gamma^{2H} \leq \frac{1}{i}$. Thus,
\begin{align}
\label{eq:final-bar-dt-grad-bound}
 \E \| \bar{d}_{t} - \nabla J_H(\theta_{t})  \|^2 &\leq \frac{1}{(N-f)t} \sum_{i=1}^t \frac{2\sigma^2}{i}  + 12 (2 L^2 + \sigma_H^2 + D_h^2 \g^{2 H}) \cdot \frac{6G_1}{\mu_F i} +  \frac{24 D_g^2}{i}  \nonumber\\
 &= \left(2\sigma^2  + 12 (2 L^2 + \sigma_H^2 + D_h^2 \g^{2 H}) \cdot \frac{6G_1}{\mu_F} +  24 D_g^2 \right)  \cdot \frac{1}{(N-f)t}\sum_{i=1}^t \frac{1}{i}  \nonumber \\
 &\leq \left(2\sigma^2  + 12 (2 L^2 + \sigma_H^2 + D_h^2 \g^{2 H}) \cdot \frac{6G_1}{\mu_F} +  24 D_g^2 \right) \cdot \left(\frac{1+ \log t}{(N-f)t}\right)\\
 &= \frac{C_2(1+ \log t)}{(N-f)t},
\end{align}
where $C_2 \coloneqq 2\sigma^2  + 12 (2 L^2 + \sigma_H^2 + D_h^2 \g^{2 H}) \cdot \frac{6G_1}{\mu_F} +  24 D_g^2$.

%We now need to find $\PP(\|d_t^{(i)} - \nabla J_H(\theta_t)\| \geq \epsilon')$. 

\subsection{Proof of Lemma \ref{lem:d_i_martingale}}
\label{subsec:d_i_martingale}

Similar to \eqref{eq:d_t-expansion}, we have
\begin{align*}
    &d_t^{(i)} - \nabla J_H(\theta_t) = (1-\eta_t)(d_{t-1}^{(i)} - \nabla J_H(\theta_{t-1})) + \eta_t \mathcal{V}_t^{(i)} + (1-\eta_t)\mathcal{W}_t^{(i)},
\end{align*}
where $\mathcal{V}_t^{(i)} = g(\tau_t^{(i)},\theta_t) - \nabla J_H(\theta_{t})$ and $\mathcal{W}_t^{(i)} = \nabla J_H(\theta_{t-1}) - \nabla J_H(\theta_{t}) + B(\hat{\tau}_t^{(i)},\hat{\theta}_t)(\theta_t - \theta_{t-1})$. Let $M_t^{(i)} \coloneqq \mathcal{V}_t^{(i)} + \left(\frac{1-\eta_t}{\eta_t}\right) \mathcal{W}_t^{(i)}$. Then,
\begin{align*}
    d_t^{(i)} - \nabla J_H(\theta_t) = (1-\eta_t)(d_{t-1}^{(i)} - \nabla J_H(\theta_{t-1})) + \eta_t M_t^{(i)}.
\end{align*}
Unrolling the above recursion gives
\begin{align*}
    d_t^{(i)} - \nabla J_H(\theta_t) &=
     \sum_{j=1}^{t} \eta_j  \left(\prod^{t}_{k=j+1}(1-\eta_k)\right) M_j^{(i)}, 
\end{align*}
where we use the convention $\prod_{k=a}^b \alpha_k = 1$ if $b<a$. 

Let $\eta_t = \frac{1}{t}$. With this choice of $\eta_t$, it follows that $\prod^{j_0+j}_{k=j_0}(1-\eta_k) = \frac{j_0-1}{j_0+j}$ and
\begin{align*}
    d_t^{(i)} - \nabla J_H(\theta_t) &=
     \sum_{j=1}^{t} \frac{1}{j}  \left(\frac{j}{t}\right) M_j^{(i)} = \frac{1}{t} \sum_{j=1}^t M_j^{(i)}.
\end{align*}

Now we show that $M_t^{(i)}$ is bounded and forms a martingale difference sequence. We have $\|\mathcal{V}_t^{(i)}\| \leq G_g$ (from Lemma \ref{lem:general-lem}) and
    \begin{align*}
        \|\mathcal{W}_t^{(i)}\| &\leq \|\nabla J_H(\theta_{t-1}) - \nabla J_H(\theta_{t}) + B(\hat{\tau}_t^{(i)},\hat{\theta}_t)(\theta_t - \theta_{t-1})\| \\
        &\leq \|\nabla J_H(\theta_{t-1}) - \nabla J_H(\theta_{t})\| + \|B(\hat{\tau}_t^{(i)},\hat{\theta}_t)(\theta_t - \theta_{t-1})\| \\
        &\leq L \|\theta_t - \theta_{t-1}\| + \|B(\hat{\tau}_t^{(i)},\hat{\theta}_t)\| \|\theta_t - \theta_{t-1}\| \leq (L+G_H)\gamma_t.
    \end{align*}
    Thus
    \begin{align}
        \|M_t^{(i)}\|& = \norm{\mathcal{V}_t^{(i)} + \left(\frac{1-\eta_t}{\eta_t}\right) \mathcal{W}_t^{(i)}}  \leq G_g + \frac{\gamma_t}{\eta_t} (L+G_H) \leq  G_g + \frac{6G_1}{\mu_F} (L+G_H) \coloneqq C_1.
    \end{align}
    Note that this implies that
    \begin{align*}
    \|d_t^{(i)} - \nabla J_H(\theta_t)\| &\leq
     \frac{1}{t} \sum_{j=1}^t \|M_j^{(i)}\| \leq C_1.
\end{align*}
    %\item Variance bound:
    %Let $\mathcal{F}_t^{(i)} \coloneqq \sigma (q_1^{(i)},\tau_1^{(i)},\hat{\tau}_1^{(i)},q_2^{(i)},\cdots,\hat{\tau}_t^{(i)})$. Then we have
    Define a sequence $\{\Tilde{M}_k^{(i)}\}_{k\geq 0}$ such that $\Tilde{M}_0^{(i)} = 0$ and $\Tilde{M}_k^{(i)} = \frac{1}{C_1}\sum_{j=0}^{k-1} M_j^{(i)}$ for $k \geq 1$. Then for all $k \geq 0$
    \begin{align*}
        \|\Tilde{M}_{k+1}^{(i)} - \Tilde{M}_k^{(i)}\| = \|M_k^{(i)}/C_1\| \leq 1.
    \end{align*}
    and
    \begin{align*}
        \E [\Tilde{M}_{k+1}^{(i)} \mid \Tilde{M}_{k}^{(i)}] &=  \E \left[ \frac{1}{C_1}\sum_{j=0}^{k} M_j^{(i)} \mid \Tilde{M}_{k}^{(i)}\right] =  \E \left[\frac{1}{C_1}\sum_{j=0}^{k-1} M_j^{(i)} + \frac{1}{C_1} \cdot M_k^{(i)}  \mid \Tilde{M}_{k}^{(i)}\right]  \\
        &=  \E \left[\Tilde{M}_{k}^{(i)} +  \frac{1}{C_1} \cdot M_k^{(i)}  \mid \Tilde{M}_{k}^{(i)}\right]  =  \Tilde{M}_{k}^{(i)} + \frac{1}{C_1} \E \left[M_k^{(i)}  \mid \Tilde{M}_{k}^{(i)}\right]. 
    \end{align*}
    %where the last equality follows from \cite{fatkhullin}.
    
    % From \cite{fatkhullin}, $\E \mathcal{V}_t$ and $\E \mathcal{W}_t$ are $0$. It follows that $\E M_t =  \E \mathcal{V}_t + \frac{1-\eta_t}{\eta_t} \E \mathcal{W}_t = 0$.

    Note that 
    \begin{align*}
        \sigma (\Tilde{M}_{k}^{(i)}) &\subset \sigma (\theta_0,\theta_1,q_1^{(i)},\tau_1^{(i)},\hat{\tau}_1^{(i)},\theta_2,\cdots,\theta_{k-1},q_{k-1}^{(i)},\tau_{k-1}^{(i)},\hat{\tau}_{k-1}^{(i)},\theta_k) \coloneqq \mathcal{F}_{k-1}^{(i)}.
    \end{align*}
    We get
    \begin{align*}
        \E \left[M_k^{(i)}  \mid \Tilde{M}_{k}^{(i)}\right] &= \E \left[ \E [M_k^{(i)}  \mid \mathcal{F}_{k-1}^{(i)}] \mid \sigma(\Tilde{M}_{k}^{(i)})\right]. \\
    \end{align*}
    Observe that 
    \begin{align*}
        \E [M_k^{(i)}  \mid \mathcal{F}_{k-1}^{(i)}] &= \E \left[\mathcal{V}_k^{(i)} + \left(\frac{1-\eta_k}{\eta_k}\right) \mathcal{W}_k^{(i)} \mid \mathcal{F}_{k-1}^{(i)}\right] \\
        &=\E \left[\mathcal{V}_k^{(i)}\mid \mathcal{F}_{k-1}^{(i)}\right] + \left(\frac{1-\eta_k}{\eta_k}\right)\E \left[ \mathcal{W}_k^{(i)} \mid \mathcal{F}_{k-1}^{(i)}\right] \\
        & = 0.
    \end{align*}
   Thus, for all $k\geq 0$
   \begin{align}
       \E [\Tilde{M}_{k+1}^{(i)} \mid \Tilde{M}_{k}^{(i)}] = \Tilde{M}_{k}^{(i)}.
   \end{align}

\subsection{Proof of Equation \eqref{eq:dt-dt-bar-bound}}
\label{subsec:equation}
\begin{lemma}[Vector Azuma-Hoeffding Inequality, \cite{Hayes2003ALI}]
\label{lem:azuma-hoeffding}
Let $\mathbf{M} = (M_0,\ldots,M_n)$ taking values in $\R^d$ be such that 
\begin{equation*}
M_0 = 0 \text{ , } \E[M_n \mid M_{n-1}] = M_{n-1} \text{ and } \|M_n - M_{n-1}\| \leq 1.
\end{equation*}
Then, for every $\delta>0$,
\begin{equation} \label{eq:final_bernstein_vector}
\PP\left(\|M_n\|\geq \delta\right) < 2e^2 e^{-\delta^2/2n}.
\end{equation}
\end{lemma}

The sequence $\{\Tilde{M}_k^{(i)}\}_{k\geq 0}$ satisfies all properties in Lemma \ref{lem:azuma-hoeffding} which gives us
\begin{align}
    \PP(\|\Tilde{M}_{n}^{(i)}\| \geq \epsilon) \leq 2e^2 e^{-\delta^2/2n}.
\end{align}

It follows that
\begin{align}
\label{eq:plugged-azuma-bound}
    \PP\left(\|d_t^{(i)} - \nabla J_H(\theta_t)\| \geq \epsilon\right) &= \PP\left(\norm{\frac{1}{t+1} \sum_{j=0}^t M_j^{(i)}} \geq \epsilon\right) = \PP\left(\frac{C_1}{t+1} \norm{\sum_{j=0}^t (M_j^{(i)}/C_1) } \geq \epsilon\right) \nonumber\\
    &= \PP\left(\frac{C_1}{t+1}\|\Tilde{M}_{t+1}^{(i)} \| \geq \epsilon\right) = \PP\left(\|\Tilde{M}_{t+1}^{(i)} \| \geq (t+1)\epsilon/C_1\right) \nonumber\\
    &< 2e^2 e^{-(t+1)\epsilon^2/2C_1^2}.
\end{align}

Plugging the above bound into \eqref{eq:f-lambda-bound-1} gives
\begin{align}
        \E\| d_t - \bar{d}_{t} \| &\leq  4e^2 C_1 \lambda (N-f) e^{-(t+1)\epsilon'^2/2C_1^2} + 2\lambda \epsilon'.
\end{align}

\subsection{Proof of Lemma \ref{lem:final-bound}}
\label{subsec:final-bound}

From \eqref{eq:fatkhullin-lemma-general}, we have
\begin{align}
\begin{split}
\label{eq:intermediate-bound}
    &J^* - J(\theta_{t+1})  \\
    &\leq \left(1 - \frac{2}{t+2} \right)  (J^* - J(\theta_{t})) + \frac{4}{3} \g_t D_g \g^{{H}} + \frac{8 \g_t }{3} \mathbb E[\|d_t - \nabla J_H(\theta_t)\|] + \frac{L \g_t^2}{2} + \frac{\varepsilon^{\prime}\g_t}{3}.
\end{split}
\end{align}

We use the following auxillary lemma to bound the above recursion 
\begin{lemma}[Lemma 12, \cite{fatkhullin}]
\label{lem:fatkhillin_aux}
    Let $\tau$ be a positive integer and $\{r_t\}_{t\geq 0}$ be a sequence of non-negative numbers such that
    \begin{align*}
        r_{t+1} \leq (1-\alpha_t)r_t + \beta_t,
    \end{align*}
where $\{\alpha_t\}_{t\geq 0}$ and $\{\beta_t\}_{t\geq 0}$ are non-negative sequences and $\alpha_t \leq 1$ for all $t$. Then for all $t_0,T \geq 1$,
\begin{align*}
    r_T \leq \frac{(t_0+\tau-1)^2 r_{t_0}}{(T+\tau-1)^2} + \frac{\sum_{t=0}^{T-1} \beta_t (t+\tau)^2}{(T+\tau-1)^2}.
\end{align*}
\end{lemma}

Using Lemma \ref{lem:fatkhillin_aux} on \eqref{eq:intermediate-bound} yields
\begin{align}
    &J^* - J(\theta_{T})  \leq \frac{J^* - J(\theta_{0})}{(T+1)^2} + \frac{\sum_{t=0}^{T-1} \beta_t (t+2)^2}{(T+1)^2},
\end{align}
where 
\begin{align*}
    \beta_t &\coloneqq \frac{4}{3} \g_t D_g \g^{{H}} + \frac{8 \g_t }{3} \mathbb E[\|d_t - \nabla J_H(\theta_t)\|] + \frac{L \g_t^2}{2} + \frac{\varepsilon^{\prime}\g_t}{3} \\
    &\leq \frac{4}{3} \g_t D_g \g^{{H}} + \frac{16C_1 \lambda \g_t }{3} \sqrt{\frac{2\log (N-f)(t+1)}{t+1}}  + \frac{64  e^2 C_1  \lambda \gamma_t}{3(t+1)} \\
    &+ \frac{8 \g_t }{3}\sqrt{\frac{C_2(1+ \log t)}{(N-f)t}}+ \frac{L \g_t^2}{2} + \frac{\varepsilon^{\prime}\g_t}{3} \\
    &= (\frac{4}{3} \g_t D_g \g^{{H}} + \frac{L \g_t^2}{2} + \frac{64  e^2 C_1  \lambda \gamma_t}{3(t+1)})+ (\frac{16C_1 \lambda \g_t }{3} \sqrt{\frac{2\log (N-f)(t+1)}{t+1}} \\
    &+  \frac{8 \g_t }{3}\sqrt{\frac{C_2(1+ \log t)}{(N-f)t}}) + \frac{\varepsilon^{\prime}\g_t}{3}.
\end{align*}

Note that
\begin{align*}
  \frac{\sum_{t=0}^{T} \varepsilon^{\prime}\g_t(t+2)^2}{3(T+1)^2} &=  \frac{\varepsilon^{\prime}\sum_{t=0}^{T} 6G_1(t+2)}{3\mu_F(T+1)^2} \\
  &=  \frac{6G_1\varepsilon^{\prime}}{3\mu_F} \frac{\sum_{t=0}^{T}(t+2)}{(T+1)^2} \\
  &= \frac{6G_1\varepsilon^{\prime}}{3\mu_F} \frac{T^2+T-1}{2(T+1)^2} \\
  &\leq \frac{G_1\varepsilon^{\prime}}{\mu_F} = \frac{\sqrt{\varepsilon_{\mathrm{bias}}}}{(1-\gamma)}. 
\end{align*}

Also, note that since $H=\frac{\log(T+1)}{1-\gamma}$, $\gamma^H \leq \frac{1}{T+1}$ and $\gamma_t = \frac{6G_1}{\mu_F (t+2)}$
\begin{align*}
   \sum_{t=0}^T \left(\frac{4 \g_t D_g \g^{{H}}}{3}  + \frac{L \g_t^2}{2} + \frac{64  e^2 C_1  \lambda \gamma_t}{3(t+1)}\right)(t+2)^2  &\leq \sum_{t=0}^T \frac{16 G_1 D_g}{\mu_F} + \frac{18 G_1^2 L}{\mu_F^2} + \frac{950 C_1 G_1 \lambda}{\mu_F} \\
   & = (T+1) \left(\frac{16 G_1 D_g}{\mu_F} + \frac{18 G_1^2 L}{\mu_F^2} + \frac{950 C_1 G_1 \lambda}{\mu_F}\right). 
\end{align*}

It follows that 
\begin{align*}
   &\frac{1}{(T+1)^2} \sum_{t=0}^T \left(\frac{4 \g_t D_g \g^{{H}}}{3}  + \frac{L \g_t^2}{2} + \frac{64  e^2 C_1  \lambda \gamma_t}{3(t+1)}\right)(t+2)^2  \\
   &\leq \left(\frac{16 G_1 D_g}{\mu_F} + \frac{18 G_1^2 L}{\mu_F^2} + \frac{950 C_1 G_1 \lambda}{\mu_F}\right) \frac{1}{T+1}.
\end{align*}

Finally,
\begin{align*}
    &\frac{16C_1 \lambda \g_t }{3} \sqrt{\frac{2\log (N-f)(t+1)}{t+1}} +  \frac{8 \g_t }{3}\sqrt{\frac{C_2(1+ \log t)}{(N-f)t}} \\
    &= \frac{16 G_1}{\mu_f (t+2)} \left(C_1 \lambda \sqrt{\frac{2\log (N-f)(t+1)}{t+1}} + \sqrt{\frac{C_2(1+ \log t)}{(N-f)t}}\right) \\
    &\leq \frac{16 G_1}{\mu_f (t+2)\sqrt{t}} \left(C_1 \lambda \sqrt{2\log (N-f)(t+1)} + \sqrt{\frac{C_2(1+ \log t)}{(N-f)}}\right).
\end{align*}

Note that
\begin{align*}
    \frac{1}{(T+1)^2}\sum_{t=0}^T \frac{(t+2)^2}{(t+2)\sqrt{t}} &= \frac{1}{(T+1)^2}\sum_{t=0}^T \frac{t+2}{\sqrt{t}} \\
    &\leq \frac{1}{(T+1)^2}\sum_{t=0}^T (\sqrt{t} + 2) \\
    &\leq \frac{ T^{3/2} + 2T}{(T+1)^2} \leq \frac{2}{(T+1)^{1/2}}.
\end{align*}

Combining all of the above, we get
\begin{align}
    J^* - J(\theta_{T})  &\leq 
    \frac{\sqrt{\varepsilon_{\mathrm{bias}}}}{(1-\gamma)}
    + \frac{J^* - J(\theta_{0})}{(T+1)^2} +\left(\frac{16 G_1 D_g}{\mu_F} + \frac{18 G_1^2 L}{\mu_F^2} + \frac{950 C_1 G_1 \lambda}{\mu_F}\right) \frac{1}{T+1} \nonumber \\
    &+ \frac{32 G_1}{\mu_F \sqrt{T+1}} \left(\sqrt{\frac{C_2(1+ \log T)}{(N-f)}} + C_1 \lambda \sqrt{2\log (N-f)(T+1)} \right).
\end{align}

From the above expression, we have
\begin{align}
    J^* - J(\theta_{T})  &\leq 
    \frac{\sqrt{\varepsilon_{\mathrm{bias}}}}{1-\gamma}
    + \mathcal{O}\bigg(\frac{G_1}{\mu_F \sqrt{T+1}} \left(\sqrt{\frac{C_2\log T}{(N-f)}} + C_1 \lambda \sqrt{\log ((N-f)(T+1))} \right)\bigg),
\end{align}
where $C_1 = G_g + \frac{6G_1}{\mu_F} (L+G_H) $ and $C_2 = 2\sigma^2  + 12 (2 L^2 + \sigma_H^2 + D_h^2 \gamma^{2 H}) \cdot \frac{6G_1}{\mu_F} +  24 D_g^2$. The Lipschitz constant $L$, variance bounds $\sigma^2$ and $\sigma_H^2$ and the remaining terms $G_g$, $G_H$, $D_h$ and $D_g$ in turn, can be bounded in terms of $\gamma$, $\mu_F$, the bound on the reward function $R$ and the Lipschitz and smoothness constants of the score function $G_1$, $G_2$ (see Lemma \ref{lem:general-lem} and \ref{masiha-lemma}). Substituting these bounds, we obtain
\begin{align}
\begin{split}
    J^* - J(\theta_{T})  &\leq 
    \frac{\sqrt{\varepsilon_{\mathrm{bias}}}}{1-\gamma}
    + \mathcal{O}\Bigg(\frac{G_1}{\mu_F \sqrt{T+1}} \Bigg(\sqrt{ \frac{\frac{G_1}{\mu_F}\cdot\left(\frac{R^2G_2^2G_1^4}{(1-\gamma)^4}+\frac{R^4(G_1^4+G_2^2)}{(1-\gamma)^8}\right)\log T}{(N-f)}} \\
    &+ \left(\frac{RG_2G_1^2}{(1-\gamma)^2}+\frac{R^2(G_1^2+G_2)}{(1-\gamma)^4}\right) \lambda \sqrt{\log ((N-f)(T+1))} \Bigg)\Bigg).
\end{split}
\end{align}

With the above bound, we obtain $J^* - J(\theta_{T})  \leq \frac{\sqrt{\varepsilon_{\mathrm{bias}}}}{1-\gamma} + \epsilon$ for
\begin{align}
\label{eq:final-bound-all-terms}
\begin{split}
T &=  \mathcal{O} \bigg( \frac{1}{\epsilon^2} \log\left(\frac{1}{\epsilon}\right) \bigg( \frac{1}{(N-f)} \cdot\left(\frac{R^2G_2^2G_1^7}{\mu_F^3(1-\gamma)^4}+\frac{R^4(G_1^7+G_1^3G_2^2)}{\mu_F^3(1-\gamma)^8}\right) \\
&+ \lambda^2 \log(N-f) \cdot \bigg(\frac{R^2G_2^2G_1^6}{\mu_F^2(1-\gamma)^4}+\frac{R^4(G_1^6+G_1^2G_2^2)}{\mu_F^2(1-\gamma)^8}\bigg)\bigg)\bigg).
\end{split}
\end{align}

\section{Compute Resources}

Experiments were conducted using the Oracle Cloud infrastructure, where each computation instance was equipped with 8 Intel Xeon Platinum CPU cores and 128 GB of memory. In Figure 1, each subfigure contains comparisons among eight algorithms each of which is repeated five times with different random seeds. The single running time for an algorithm is approximately two hours. Thus, to reproduce Figure 1, it requires approximately 480 CPU hours. Similarly, for Figure 2, the time would be about 120 CPU hours.

\end{document}